\documentclass{article}

\usepackage[square,numbers]{natbib}
\usepackage[final]{neurips_2021}

\usepackage[utf8]{inputenc} % allow utf-8 input
\usepackage[T1]{fontenc}    % use 8-bit T1 fonts
\usepackage{hyperref}       % hyperlinks
\usepackage{url}            % simple URL typesetting
\usepackage{booktabs}       % professional-quality tables
\usepackage{amsfonts}       % blackboard math symbols
\usepackage{nicefrac}       % compact symbols for 1/2, etc.
\usepackage{microtype}      % microtypography
\usepackage{xcolor}         % colors
\usepackage{graphicx}       % figures
\usepackage{pict2e}
\usepackage{xcolor}
\usepackage{wrapfig}

\usepackage{amssymb,amsmath,amsthm}
\theoremstyle{plain}
\newtheorem{thm}{\protect\theoremname}
\theoremstyle{plain}
\newtheorem{assumption}{\protect\assumptionname}
\theoremstyle{plain}
\newtheorem{prop}{\protect\propositionname}
\theoremstyle{plain}
\newtheorem{lem}{\protect\lemmaname}

\providecommand{\assumptionname}{Assumption}
\providecommand{\lemmaname}{Lemma}
\providecommand{\propositionname}{Proposition}
\providecommand{\theoremname}{Theorem}

\newcommand{\+}[1]{\mathbf{#1}}

\def\R{\mathbb{R}}

\usepackage{mathtools}

\usepackage{algorithm}
\usepackage[noend]{algpseudocode}

\usepackage{enumerate}

\title{Matrix factorisation and the interpretation of geodesic distance}

\author{%
  Nick Whiteley \\
  University of Bristol\\
  \texttt{nick.whiteley@bristol.ac.uk}
  \And
  Annie Gray \\
  University of Bristol\\
  \texttt{annie.gray@bristol.ac.uk}
  \And
  Patrick Rubin-Delanchy\\
  University of Bristol\\
  \texttt{patrick.rubin-delanchy@bristol.ac.uk}
}

\begin{document}

\maketitle

\begin{abstract} 
    Given a graph or similarity matrix, we consider the problem of recovering a notion of true distance between the nodes, and so their true positions. We show that this can be accomplished in two steps: matrix factorisation, followed by nonlinear dimension reduction. This combination is effective because the point cloud obtained in the first step lives close to a manifold in which latent distance is encoded as geodesic distance. Hence, a nonlinear dimension reduction tool, approximating geodesic distance, can recover the latent positions, up to a simple transformation. We give a detailed account of the case where spectral embedding is used, followed by Isomap, and provide encouraging experimental evidence for other combinations of techniques.
\end{abstract}

\section{Introduction} \label{sec:introduction}
Assume we observe, or are given as the result of a computational procedure, data in the form of a symmetric matrix $\+A\in \R^{n \times n}$ which we relate to unobserved vectors $Z_1, \ldots, Z_n \in \mathcal{Z} \subset \R^d$, where $d\ll n$, by
\begin{equation}
\+A^{ij} = f(Z_i, Z_j) + \+E^{ij},\quad 1\leq i \leq j \leq n,\label{eq:model}
\end{equation}
for some real-valued symmetric function $f$ which will be called a kernel, and a matrix of unobserved perturbations, $\+E \in \R^{n \times n}$. As illustrative examples, $\+A$ could be the observed adjacency matrix of a graph with $n$ vertices, a similarity matrix associated with $n$ items, or some matrix-valued estimator of covariance or correlation between $n$ variates. 

Broadly stated, our goal is to recover $Z_1, \ldots, Z_n$ given $\+A$, up to identifiability constraints. We seek a method to perform this recovery which can be implemented \emph{in practice} without any knowledge of $f$, nor of any explicit statistical model for $\+A$, and to which we can plug in various matrix-types, be it adjacency, similarity, covariance, or other.

At face value, this may seem rather a lot to ask. Without assumptions, neither  $f$, the $Z_i$'s, nor even $d$ are identifiable \cite{janson2021can}. However, under quite general assumptions, we will show that a practical solution is provided by the key combination of two procedures: i) matrix factorisation of $\+A$, such as spectral embedding, followed by ii) nonlinear dimension reduction, such as Isomap \cite{tenenbaum2000global}. 

This combination is effective because of the following facts, only the first of which is already known. The matrix factorisation step approximates high-dimensional images of the $Z_i$'s, living near a $d$-dimensional manifold \cite{rubin2020manifold}. Under regularity and non-degeneracy assumptions on the kernel, which for our analysis is taken to be positive definite, geodesic distance in this manifold equals Euclidean geodesic distance on $\mathcal{Z}$, up to simple transformations of the coordinates, for example scaling. Thus a dimension reduction technique which approximates in-manifold geodesic distances can be expected to recover the $Z_i$'s, up to a simple transformation. If those assumptions fail, as they must with real data, we may still find that those approximate geodesic distances are useful because they reflect a geometry implicit in the kernel. 

There are many works which address recovery of the $Z_i$'s when one \emph{does} consider a particular matrix-type, an explicit statistical model, and/or specific form for $f$. When $\+A$ is an adjacency matrix and $f(Z_i,Z_j)$ is the probability of an edge between nodes $i$ and $j$, the construct \eqref{eq:model} is a latent position model \citep{hoff2002latent} and various scenarios have been studied: $\mathcal{Z}$ is a sphere and $f(x,y)$ is a function of $\langle x,y\rangle$ \citep{NEURIPS2019_c4414e53};  estimation by graph distance when $f(x,y)$ is a function of $\|x-y\|$ \citep{10.1214/21-EJS1801}; spectral estimation when $f(x,y)$ is a possibly indefinite inner-product \citep{sussman2012consistent,tang2016limit,athreya2017statistical,rubin2017statistical}; inference under a logistic model via MCMC \citep{hoff2002latent} or variational approximation \citep{salter2013variational}. The case where $\mathcal{Z}$ is a discrete set of points corresponds to the very widely studied stochastic block model \citep{holland1983stochastic}, see \citep{abbe2017community,lei2015consistency,rohe2011spectral,sussman2012consistent,amini2013pseudo,bickel2013asymptotic} and references therein; by contrast the methods we propose are directed at the case where  $\mathcal{Z}$ is continuous. The opposite problem of estimating $f$, with $Z_i$ unknown but assumed uniform on $[0,1]$, is known as graphon estimation \cite{gao2015rate}.

Concerning the case when $\+A$ is a covariance matrix, Latent Variable Gaussian Process models \citep{lawrence2003gaussian,lawrence2005probabilistic,titsias2010bayesian} use $f(Z_i, Z_j)$ to define the population covariance between the $i$th and $j$th variates under a hierarchical model. When $f(x,y)=\langle x,y\rangle$, this reduces to a latent variable optimisation counterpart of Probabilistic PCA \citep{tipping1999probabilistic} and the maximum likelihood estimate of the $Z_i$'s is obtained from the eigendecomposition of the empirical covariance matrix, which in our setting could be $\+A$. When $f$ is nonlinear, but fixed e.g. to a Radial Basis Function kernel, gradient and variational techniques are available \citep{lawrence2003gaussian,lawrence2005probabilistic,titsias2010bayesian}. See the same references for onward connections to kernel PCA \citep{scholkopf1998nonlinear}.

{\bf Our contribution and precursors.} The general practice of using dimensionality reduction and geodesic distance to extract latent features from data is common in data science and machine-learning. Examples arise in genomics \citep{margaryan2020population}, neuroscience \citep{yamin2019comparison}, speech analysis \citep{karam10_interspeech,hasan2017word} and cyber-security \citep{anglade:hal-02316303}. What sets our work apart from these contributions is that we establish a new rigorous basis for this practice.

It has been understood for several years that the high-dimensional embedding obtained by matrix factorisation can be related via a feature map to the $Z_i$ of model \eqref{eq:model} \citep{kallenberg1989representation,bollobas2007phase,lovasz2012large,tang2013universally,xu2018rates,lei2018network}. However, it was only recently observed that the embedding must therefore concentrate about a low dimensional set, in the Hausdorff sense \cite{rubin2020manifold}. %something about bi-lipschitz
Our key mathematical contribution is to describe the topology and geometry of this set, proving it is a topological manifold and establishing how in-manifold geodesic distance is related to geodesic distance in $\mathcal{Z}$. Riemannian geometry underlying kernels was sketched in \citep{amari1999improving} but without consideration of geodesic distances or rigorous proofs, and not in the context of latent position estimation.   The work \citep{athreya2018estimation,trosset2020learning} was an inspiration to us, suggesting Isomap as a tool for analysing spectral embeddings, under a latent structure model in which $\mathcal{Z}$ is a one-dimensional curve in $\R^d$ and $f$ is the inner product. A key feature of our problem setup is that $f$ is unknown, in which case the manifold is not available in closed form and typically lives in an infinite-dimensional space. To our knowledge, we are the first to show why spectral embedding followed by Isomap might recover the true $Z_i$'s, or a useful transformation thereof, in general. 

We complement our mathematical results with experimental evidence, obtained from both simulated and real data, suggesting that alternative combinations of matrix-factorisation and dimension-reduction techniques work too. For the former, we consider the popular node2vec \citep{grover2016node2vec} algorithm, a likelihood-based approach for graphs said to perform matrix factorisation implicitly \cite{qiu2018network,zhang2021consistency}. For the latter, we consider the popular t-SNE \cite{maaten2008visualizing} and UMAP \cite{mcinnes2018umap} algorithms. As predicted by the theory, a direct low-dimensional matrix factorisation, whether using spectral embedding or node2vec, is less successful.

\section{Proposed methods and their rationale}
\begin{figure}
\includegraphics[width=\textwidth]{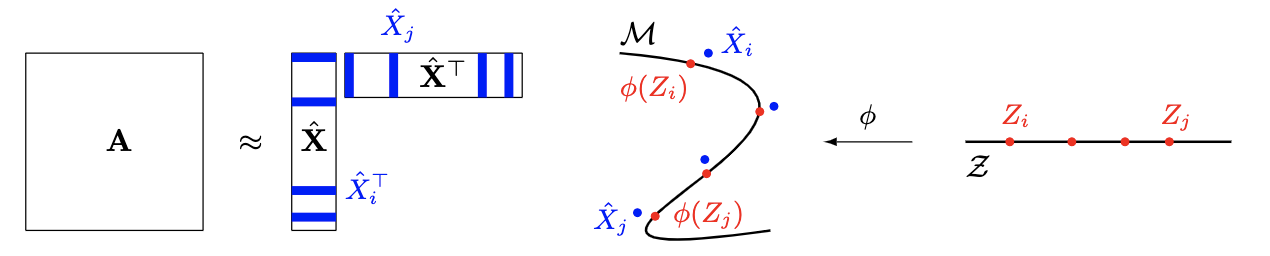}
  \caption{Illustration of theory in the case $d=1$. Our analysis reveals how geodesic distance, along $\mathcal{M}$, between $\phi(Z_i)$ and $\phi(Z_j)$, is related to geodesic distance, along $\mathcal{Z}$, between $Z_i$ and $Z_j$.} \label{fig:illustration}
\end{figure}

\subsection{Spectral embedding, as estimating $\phi(Z_i)$} \label{sec:spec_embed}
In our mathematical setup (precise details come later) the kernel $f$ will be nonnegative definite and $\phi$ will be the associated Mercer feature map. It is well known that each point $\phi(Z_i)$ then lives in an infinite-dimensional Hilbert space, which will be denoted $\ell_2$, and the inner-product $\langle \phi(Z_i),\phi(Z_j)\rangle_2$ in this space equals  $f(Z_i,Z_j)$.

{\bf The spectral embedding procedure.} For $p \leq n$, we define the $p$-dimensional spectral embedding of $\+A$ to be  $\hat{ \+X}=[\hat X_1, \ldots, \hat X_n]^\top= \hat{\+U} |\hat{\+S}|^{1/2} \in \R^{n \times p}$, where $|\hat{\+S}| \in \R^{p \times p}$ is a diagonal matrix containing the absolute values of the $p$ largest eigenvalues of $\+A$, by magnitude, and $\hat{\+U }\in \R^{n \times p}$ is a matrix containing corresponding orthonormal eigenvectors, in the same order. The R packages irlba and RSpectra provide fast solutions which can exploit sparse inputs.

One should think of $\hat X_i$ as approximating the vector of first $p$ components of $\phi(Z_i)$, denoted $\phi_p(Z_i)$, up to orthogonal transformation, and this can be formalised to a greater or lesser extent depending on what assumptions are made. There are several situations, e.g. $f$ any polynomial \cite{rubin2020manifold}, the cosine kernel used in Section~\ref{sec:simulation}, the degree-corrected \cite{karrer2011stochastic} or mixed-membership \cite{airoldi2008mixed} stochastic block model, in which only the first $p_0$ (say) components of $\phi(\cdot)$ are nonzero, where typically $p_0 \geq d$. If, after $n$ reaches $p_0$, we embed into $p=p_0$ dimensions, then with $\|\cdot\|$ denoting the Euclidean norm, we have \cite{gallagher2019spectral}:
\begin{equation}
\max_{i \in \{1,\ldots,n\}} \lVert \+Q \hat X_i - \phi_p(Z_i) \rVert = O_{\mathbb{P}}\left(\tfrac{(\log n)^{c}}{n^{1/2}}\right),\label{eq:consistency_simple}
\end{equation}
for a universal constant  $c\geq 1$, orthogonal matrix  $\+Q$, under regularity assumptions on the $Z_i$'s, $f$ and $\+E$ (that $Z_i$ are i.i.d., $f(Z_i, Z_j)$ has finite expectation, and the perturbations $\+E^{ij}$ are independent and centered with exponential tails). This encompasses the case where $\+A$ is binary, for example a graph adjacency matrix \cite{lyzinski2017community,rubin2017statistical}. Similar results are available in the cases where $\+A$ is a Laplacian \cite{rubin2017statistical,modell2021spectral} or covariance matrix \cite{two_to_infinity}. The methods of this paper are based, in practice, on the distances $\lVert \hat X_i - \hat X_j \rVert$, which are invariant to orthogonal transformation and so for the purposes of validating $\lVert \hat X_i - \hat X_j \rVert\approx\lVert\phi_p(Z_i)-\phi_p(Z_j)\rVert$  the presence of $\+Q$ in \eqref{eq:consistency_simple} is immaterial. 

For $\hat X_i$ to converge to $\phi(Z_i)$ more generally, we must let its dimension $p$ grow with $n$ and, at least given the present state of literature, accept weaker consistency results, for example, convergence in Wasserstein distance between $\+Q \hat X_1, \ldots \+Q\hat X_n$ and $\phi_p(Z_1), \ldots, \phi_p(Z_n)$ \citep{lei2018network}. Uniform consistency results, in the style of \eqref{eq:consistency_simple}, are also available for indefinite \cite{rubin2017statistical}, bipartite, and directed graphs \cite{jones2020multilayer}. These are left for future work because of the complications of $\+Q$ no longer being orthogonal.

{\bf Rank selection.} In real data, where $n$ is typically fixed, there is no `best' way of selecting $p$, as discussed for example in \cite{priebe2019two}. The method of \citep{zhu2006automatic}, based on profile-likelihood, provides a popular, practical choice, taking as input the spectrum of $\+A$, and is implemented in the R package igraph.

\subsection{Isomap, as estimating $Z_i$} \label{sec:isomap}
We propose to recover $Z_1, \ldots, Z_n$ through the following procedure.
 \begin{algorithm}[H]
 \caption{Isomap procedure}\label{alg:spec}
\begin{algorithmic}[1]
  \Statex \textbf{input} $p$-dimensional points $\hat X_1, \ldots, \hat X_n$
  \State Compute the neighbourhood graph of radius $\epsilon$: a weighted graph connecting $i$ and $j$, with weight $\|\hat X_i - \hat X_j\|$, if $\|\hat X_i - \hat X_j\| \leq \epsilon$
  \State Compute the matrix of shortest paths on the neighbourhood graph, $\hat{\+D}_{\mathcal{M}} \in \R^{n \times n}$
  \State Apply classical multidimensional scaling (CMDS) to $\hat{\+D}_{\mathcal{M}}$ into $\R^d$
  \Statex \Return $d$-dimensional points $\hat Z_1, \ldots, \hat Z_n$
\end{algorithmic}
 \end{algorithm}
 
Theorem~\ref{thm:generic} below establishes, under general assumptions, that $\mathcal{M} \coloneqq \phi(\mathcal{Z})$ is a topological manifold of Hausdorff dimension exactly that of $\mathcal{Z}$ (as opposed to an upper bound \cite{rubin2020manifold}). It also proposes a `change of metric', often trivial, under which a path on $\mathcal{Z}$ and its image on $\mathcal{M}$ have the same length. 

This result explains how $\hat Z_1, \ldots, \hat Z_n$ can estimate $Z_1, \dots, Z_n$. First, think of a path from $\hat X_i$ to $\hat X_j$ on the neighbourhood graph as a noisy, discrete version of some corresponding continuous path from $\phi(Z_i)$ to $\phi(Z_j)$ on $\mathcal{M}$. The length of the first (the sum of the weights of the edges) is approximately equal to the length of the second if measured in the standard metric: an infinitesmal step from $x$ to $x + \mathrm{d}x$ has length $\smash{\langle \mathrm{d}x,\mathrm{d}x\rangle_2^{1/2}}$. 
 By inversion of $\phi$ we can trace a third path, taking us from $Z_i$ to $Z_j$ on $\mathcal{Z}$. But to make its length agree with the first two, we must pick a non-standard metric: an infinitesmal step from $z$ to $z + \mathrm{d}z$ must be regarded to have length $\langle \mathrm{d}z,\+H_{z} \mathrm{d}z\rangle^{1/2} = \sqrt{\mathrm{d}z^\top \+H_{z} \mathrm{d}z}$, where
 \[\+H_{z}^{ij}\coloneqq\left.\dfrac{\partial^{2}f}{\partial x^{(i)}\partial y^{(j)}}\right|_{(z,z)}, \quad i,j \in \{1, \ldots, d\}.\]
 The exciting news for practical purposes is that, through elementary calculus, we might establish that $\+H_{z}$ is constant in $z$ (e.g. if $f$ is translation-invariant) or even proportional to the identity (e.g. if $f$ is just a function of Euclidean distance). In the latter case, if $\mathcal{Z}$ is convex, it is not too difficult to see that the length of the \emph{shortest} path, from $\phi(Z_i)$ to $\phi(Z_j)$ on $\mathcal{M}$, must be proportional to the Euclidean distance between $Z_i$ and $Z_j$. Thus, the matrix of shortest paths obtained by Isomap (Step 2) approximates the matrix of Euclidean distances between the $Z_i$ (up to scaling), from which the $Z_i$ themselves can be recovered (Step 3), up to scaling, rotation, and translation. Of course, we have taken a few liberties in this argument, such as assuming $\phi$ to be invertible and $\mathrm{d}z^\top \+H_{z} \mathrm{d}z$ to be positive. We will address these rigorously in the next section.
 
{\bf Dimension and radius selection.} For visualisation in our examples we will pick $d=1$ or $d=2$. For other applications, $d$ can be estimated as the approximate rank of $\hat{\+D}_{\mathcal{M}}$ after double-centering \cite{torgerson1952multidimensional}, e.g. via \citep{zhu2006automatic}, again. In practice, we suggest picking $\epsilon$ just large enough for the neighbourhood graph to be connected or, if the data have outliers, a fixed quantile (such as 5\%) of $\hat{\+D}_{\mathcal{M}}$, removing nodes outside the largest connected component. The same recommendations apply if the $k$-nearest neighbour graph is used instead of the $\epsilon$-neighbourhood graph.

% (double-centering, followed by spectral embedding \cite{torgerson1952multidimensional}) 

\section{Theory}\label{sec:theory}

\subsection{Setup and assumptions}\label{sec:setup}
The usual inner-product and Euclidean norm on $\mathbb{R}^{d}$ are
denoted $\left\langle \cdot,\cdot\right\rangle $ and $\|\cdot\|$.
$\ell_{2}$ is the set of $x=[x^{(1)}\,x^{(2)}\,\cdots]^{\top}\in\mathbb{R}^{\mathbb{N}}$
such that $\|x\|_{2}\coloneqq(\sum_{k=1}^{\infty}|x^{(k)}|^{2})^{1/2}<\infty$.
For $x,y\in\ell_{2}$ , $\left\langle x,y\right\rangle _{2}\coloneqq\sum_{k=1}^{\infty}x^{(k)}y^{(k)}$. With $d\geq1$, let $\mathcal{Z}$ be a compact subset of $\mathbb{R}^{d}$,
and let $\widetilde{\mathcal{Z}}$ be a closed ball in $\mathbb{R}^{d}$,
centered at the origin, such that $\mathcal{Z}\subset\widetilde{\mathcal{Z}}$.
Let $f:\widetilde{\mathcal{Z}}\times\widetilde{\mathcal{Z}}\to\mathbb{R}$
be a symmetric, continuous, nonnegative-definite function.

By Mercer's Theorem, e.g., \cite[Thm 4.49]{steinwart2008support},
there exist nonnegative real numbers $(\lambda_{k})_{k\geq1}$ and
functions $(u_{k})_{k\geq1}$, with each $u_{k}:\widetilde{\mathcal{Z}}\to\mathbb{R}$, which are orthonormal with respect to the inner-product $(u_j,u_k)\mapsto\int_{\tilde{\mathcal{Z}}}u_j(x)u_k(x)\mathrm{d}x$

and such that 
\begin{equation}
f(x,y)=\sum_{k=1}^{\infty}\lambda_{k}u_{k}(x)u_{k}(y),\quad x,y\in\widetilde{\mathcal{Z}},\label{eq:mercer}
\end{equation}
where the convergence is absolute and uniform. For $x\in\mathcal{Z}$ let $\phi(x)\coloneqq[\lambda_{1}^{1/2}u_{1}(x)\;\lambda_{2}^{1/2}u_{2}(x)\;\cdots]^{\top}\in\mathbb{R}^{\mathbb{N}}$.
The image of $\mathcal{Z}$ by $\phi$ is denoted $\mathcal{M}$.
Observe from (\ref{eq:mercer}) that for any $x,y\in\mathcal{Z}$,
$f(x,y)=\left\langle \phi(x),\phi(y)\right\rangle _{2}$ , $\|\phi(x)\|_{2}^{2}=f(x,x)$,
and the latter is finite for any $x\in\mathcal{Z}$ since $f$ is continuous and $\mathcal{Z}$ is compact, hence $\mathcal{M}\subset\ell_{2}$.

The following definitions are standard in metric geometry \cite{burago2001course}. For any $a,b\in\ell_{2},$ a \emph{path} in $\mathcal{M}$ with end-points
$a,b$ is a continuous function $\gamma:[0,1]\to\mathcal{M}$ such
that $\gamma_{0}=a$ and $\gamma_{1}=b$. With $n\geq1$, a non-decreasing
sequence $t_{0},\ldots,t_{n}$ such that $t_{0}=0$ and $t_{n}=1$,
is called a \emph{partition.} Given a path $\gamma$ and a partition
$\mathcal{P}=(t_{0},\ldots,t_{n})$, define $\chi(\gamma,\mathcal{P})\coloneqq\sum_{k=1}^{n}\|\gamma_{t_{k}}-\gamma_{t_{k-1}}\|_{2}$.
The \emph{length} of $\gamma$ (w.r.t. $\|\cdot\|_{2}$) is $l(\gamma)\coloneqq\sup_{\mathcal{P}}\chi(\gamma,\mathcal{P})$,
where the supremum is over all possible partitions. The geodesic distance in $\mathcal{M}$ between $a$ and $b$ is defined to be the infimum of $l(\gamma)$ over all paths $\gamma$ in $\mathcal{M}$ with end-points $a,b$.  $\+D_{\mathcal{M}}^{ij}$ denotes geodesic distance in $\mathcal{M}$ between $\phi(Z_i)$ and $\phi(Z_j)$.
Similarly a path $\eta$ in $\mathcal{Z}$ with end-points $x,y$
is a continuous function $\eta:[0,1]\to\mathcal{Z}$ such that $\eta_{0}=x,\eta_{1}=y$,
and with $\chi(\eta,\mathcal{P})\coloneqq\sum_{k=1}^{n}\|\eta_{t_{k}}-\eta_{t_{k-1}}\|$
the length of $\eta$ (w.r.t. $\|\cdot\|$) is $l(\eta)\coloneqq\sup_{\mathcal{P}}\chi(\eta,\mathcal{P})$. $\+D_{\mathcal{Z}}^{ij}$ denotes geodesic distance in $\mathcal{Z}$ between $Z_i$ and $Z_j$. If $\mathcal{Z}$ is convex, $\+D_{\mathcal{Z}}^{ij}=\|Z_i-Z_j\|$.

\begin{assumption}
\label{assu:injective}For all $x,y\in\mathcal{Z}$ with $x\neq y$,
there exists $a\in\mathcal{Z}$ such that $f(x,a)\neq f(y,a)$.
\end{assumption}

\begin{assumption}
\label{assu:c2_etc}f is $C^{2}$ on $\widetilde{\mathcal{Z}}$ and
for every $z\in\mathcal{Z}$, the matrix $\+H_{z}$ is positive definite. 
\end{assumption}
Assumption~\ref{assu:injective} is used to show $\phi$ is injective. Assumption~\ref{assu:c2_etc} has various implications, loosely speaking it ensures a non-degenerate relationship between path-length in $\mathcal{M}$ and path-length in $\mathcal{Z}$. Concerning our general setup, if $f$ is not required to be nonnegative definite, but is still symmetric, then a representation formula like \eqref{eq:mercer} is available under e.g. trace-class assumptions \citep{rubin2020manifold}.   It is of interest to generalise our results to that scenario but technical complications are involved, so we leave it for future research.

\subsection{Path-lengths in $\mathcal{M}$ and in $\mathcal{Z}$}
\begin{thm}
\label{thm:generic}Let assumptions \ref{assu:injective} and \ref{assu:c2_etc}
hold. Then $\phi$ is a bi-Lipschitz homeomorphism between $\mathcal{Z}$
and $\mathcal{M}$. Let $a,b$ be any two points in $\mathcal{M}$ such
that there exists a path in $\mathcal{M}$ with end-points $a,b$
of finite length. Let $\gamma$ be any such path and define $\eta:[0,1]\to\mathcal{Z}$
by $\eta_{t}\coloneqq\phi^{-1}(\gamma_{t})$. Then $\eta$ is a path
in $\mathcal{Z}$ with $l(\eta)<\infty$. For any $\epsilon>0$ there
exists a partition $\mathcal{P}_{\epsilon}$ such that for any partition
$\mathcal{P}=(t_{0},\ldots,t_{n})$ satisfying $\mathcal{P}_{\epsilon}\subseteq\mathcal{P}$,
\begin{equation}
\left|l(\gamma)-\sum_{k=1}^{n}\left\langle \eta_{t_{k}}-\eta_{t_{k-1}},\+H_{\eta_{t_{k-1}}}(\eta_{t_{k}}-\eta_{t_{k-1}})\right\rangle ^{1/2}\right|\leq\epsilon.\label{eq:path_length}
\end{equation}
If $\gamma$ and $\eta$ are continuously differentiable, the following
two equalities hold:
\begin{equation}
l(\gamma)=\int_{0}^{1}\|\dot{\gamma}_{t}\|_{2}\mathrm{d}t=\int_{0}^{1}\left\langle \dot{\eta}_{t},\+H_{\eta_{t}}\dot{\eta}_{t}\right\rangle ^{1/2}\mathrm{d}t.\label{eq:riemannian_path_length}
\end{equation}
\end{thm}

The proof is in the appendix. In the terminology of differential geometry, the collection
of inner-products $(x,y)\mapsto\left\langle x,\+H_{z}y\right\rangle $,
indexed by $z\in\mathcal{Z}$, constitute a \emph{Riemannian metric
}on $\mathcal{Z}$, see e.g. \cite[Ch.1 and p.121]{petersen2006riemannian} for background,
and the right hand side of (\ref{eq:riemannian_path_length}) is
the length of the path $\eta$ in $\mathcal{Z}$ with respect to this
Riemannian metric. The theorem tells us that finding geodesic distance in $\mathcal{M}$,
i.e. minimising $l(\gamma)$ with respect to $\gamma$ for fixed end-points
say $a=\phi(Z_i),b=\phi(Z_j)$, is equivalent to minimizing path-length
in $\mathcal{Z}$ between end-points $\phi^{-1}(a)=Z_i,\phi^{-1}(b)=Z_j$ under the Riemannian metric.
In section \ref{sec:shortest_paths} we will show how, for various classes of kernels,
this minimisation can be performed in closed form leading to explicit relationships between $\+D_\mathcal{M}^{ij}$ and $\+D_{\mathcal{Z}}^{ij}$. In Section~\ref{sec:shortest_paths} we focus on \eqref{eq:riemannian_path_length} rather than \eqref{eq:path_length} only for ease of exposition, it is shown in the appendix that exactly the same conclusions can be derived directly from \eqref{eq:path_length}.   To avoid repetition Assumption~\ref{assu:injective} will be taken to hold throughout section \ref{sec:shortest_paths} without further mention.

\subsection{Geodesic distance in $\mathcal{M}$ and in $\mathcal{Z}$}\label{sec:shortest_paths}
{\bf Translation invariant kernels.} Suppose $\mathcal{Z}\subset{\mathbb{R}^d}$ is compact and convex, and $f(x,y)=g(x-y)$ where $g:\mathbb{R}^{d}\to\mathbb{R}$
is $C^{2}$. In this situation $\+H_{z}$ is constant in $z$, and
equal to the Hessian of $-g$ evaluated at the origin. The positive-definite
part of Assumption~\ref{assu:c2_etc} is therefore satisfied if and
only if $g$ has a local maximum at the origin. We now use Theorem \ref{thm:generic} to find geodesic distance in $\mathcal{M}$ between generic  points $a,b\in\mathcal{M}$
for this class of translation-invariant kernels. To do this we obtain
a lower bound on $l(\gamma)$ over all paths $\gamma$ in $\mathcal{M}$ with end-points
$a,b$, then show there exists such a path whose length achieves this lower
bound.

Let $\+G=\+V\+\Lambda^{1/2}$ where $\+V\+\Lambda \+V^{\top}$ is the eigendecomposition
of the Hessian of $-g$ evaluated at the origin, so $\+H_{z}=\+G\+G^{\top}$
for all $z$. Then from (\ref{eq:riemannian_path_length}), for a generic path $\gamma$ in $\mathcal{M}$ with end-points $a,b$,
\begin{align}
l(\gamma) & =\int_{0}^{1}\left\langle \dot{\eta}_{t},\+H_{\eta_{t}}\dot{\eta}_{t}\right\rangle ^{1/2}\mathrm{d}t=\int_{0}^{1}\left\langle \dot{\eta}_{t},\+G\+G^{\top}\dot{\eta}_{t}\right\rangle ^{1/2}\mathrm{d}t=\int_{0}^{1}\|\+G^{\top}\dot{\eta}_{t}\|\mathrm{d}t\label{eq:translation-invariant}\\
 & \geq\left\Vert \int_{0}^{1}\+G^{\top}\dot{\eta}_{t}\mathrm{d}t\right\Vert =\left\Vert \+G^{\top}\int_{0}^{1}\dot{\eta}_{t}\mathrm{d}t\right\Vert =\|\+G^{\top}(\eta_{1}-\eta_{0})\|=\|\+G^{\top}[\phi^{-1}(b)-\phi^{-1}(a)]\|,\label{eq:translation_invariant_lower_bound}
\end{align}
where the inequality is due to the triangle inequality for the $\|\cdot\|$
norm. The right-most term in (\ref{eq:translation_invariant_lower_bound})
is independent of $\gamma$,
other than through the end-points $a,b.$ To see there exists a path whose length equals this lower bound, take
\begin{equation}
\tilde{\eta}_{t}\coloneqq\phi^{-1}(a)+t[\phi^{-1}(b)-\phi^{-1}(a)],\quad t\in[0,1],\label{eq:straight-line}
\end{equation}
which is well-defined as a path in $\mathcal{Z}$, since for this class of kernels $\mathcal{Z}$ is assumed convex. With $\tilde{\gamma}_{t}\coloneqq\phi(\tilde{\eta}_{t})$,
$\tilde{\gamma}$ is clearly a path in $\mathcal{M}$. Differentiating $\tilde{\eta}_{t}$ w.r.t.
$t$ and substituting into (\ref{eq:translation-invariant}) shows $l(\tilde{\gamma})=\|\+G^{\top}[\phi^{-1}(b)-\phi^{-1}(a)]\|$, as required. The following proposition summarises our conclusions in the case that $a=\phi(Z_i)$ and $b=\phi(Z_j)$, for any $Z_i,Z_j$.
\begin{prop}\label{prop:translation_invariant}
If $\mathcal{Z}$ is compact and convex, and $f(x,y)=g(x-y)$ where $g:\mathbb{R}^{d}\to\mathbb{R}$
is $C^{2}$ and has a local maximum at the origin, then $\+D_\mathcal{M}^{ij}$ is equal to Euclidean
distance between $Z_i$ and $Z_j$,
up to their linear transformation by $\+G^{\top}$. In the particular case where $g(x-y)=h(\|x-y\|^2)$ with $h^\prime(0)<0$, we have  $\+D_\mathcal{M}^{ij}\propto \+D_{\mathcal{Z}}^{ij} = \|Z_i-Z_j\|$.
\end{prop}

{\bf Inner-product kernels.} Suppose $\mathcal{Z}=\{x\in\mathbb{R}^{d}:\|x\|=1\}$ and 
$f(x,y)=g(\left\langle x,y\right\rangle )$, where $g:\mathbb{R}\to\mathbb{R}$
is $C^{2}$ and such that $g^{\prime}(1)>0$. In this case $\+H_{z}=g^{\prime}(1)\+I + g^{\prime\prime}(1)zz^{\top}$,
and by a result of \cite{kar2012random}, any kernel of the form $f(x,y)=g(\left\langle x,y\right\rangle )$
on the sphere is positive-definite iif $g(x)=\sum_{n=0}^{\infty}a_{n}x^{n}$
for some nonnegative $(a_{n})_{n\geq0}$, implying $g^{\prime\prime}(1)\geq0$.
Assumption~\ref{assu:c2_etc} then holds. To derive the geodesic distance in $\mathcal{M}$,
first write out (\ref{eq:riemannian_path_length}) for a generic path $\gamma$ in $\mathcal{M}$ with end-points $a,b$:
\begin{equation}
l(\gamma)=\int_{0}^{1}\left\langle \dot{\eta}_{t},\+H_{\eta_{t}}\dot{\eta}_{t}\right\rangle ^{1/2}\mathrm{d}t=\int_{0}^{1}\left(g^{\prime}(1)\|\dot{\eta}_{t}\|^{2}+g^{\prime\prime}(1)\left|\left\langle \eta_{t},\dot{\eta}_{t}\right\rangle \right|^{2}\right)^{1/2}\mathrm{d}t.\label{eq:inner_prod_path_length}
\end{equation}
Since $\mathcal{Z}$ is a radius-$1$ sphere centered at the origin we must have $\|\eta_{t}\|=1$ for all $t$,
so $0=\frac{1}{2}\frac{\mathrm{d}}{\mathrm{d}t}\|\eta_{t}\|^{2}=\frac{1}{2}\frac{\mathrm{d}}{\mathrm{d}t}\left\langle \eta_{t},\eta_{t}\right\rangle =\left\langle \eta_{t},\dot{\eta}_{t}\right\rangle $.
Therefore the r.h.s. of (\ref{eq:inner_prod_path_length}) is in fact
equal to $g^{\prime}(1)^{1/2}\int_{0}^{1}\|\dot{\eta}_{t}\|\mathrm{d}t$.
This is minimised over all possible paths $\eta$ in $\mathcal{Z}$ with end-points
$\phi^{-1}(a),\phi^{-1}(b)$
when $\eta$ is a shortest (with respect to Euclidean distance) circular arc in $\mathcal{Z}$, in
which case $\int_{0}^{1}\|\dot{\eta}_{t}\|\mathrm{d}t=\arccos\left\langle \phi^{-1}(a),\phi^{-1}(b)\right\rangle $.
Thus we have:
\begin{prop}\label{prop:inner_prod_kernels}
If $\mathcal{Z}=\{x\in\mathbb{R}^{d}:\|x\|=1\}$ and $f(x,y)=g(\left\langle x,y\right\rangle )$ where $g:\mathbb{R}\to\mathbb{R}$ is
$C^{2}$ and such that $g^{\prime}(1)>0$, then $\+D_\mathcal{M}^{ij}=g^{\prime}(1)^{1/2}\+D_{\mathcal{Z}}^{ij}=g^{\prime}(1)^{1/2}\arccos\left\langle Z_i,Z_j\right\rangle$.
\end{prop}

{\bf Additive kernels.}
Suppose $\mathcal{Z}$ is the Cartesian product of intervals 
$\mathcal{Z}_{i}=[c_i^-,c_i^+]$, $i=1,\ldots,d$ and $f(x,y)=\sum_{i=1}^{d}\alpha_{i}f_{i}(x^{(i)},y^{(i)})$,
$x=[x^{(1)}\,\cdots\,x^{(d)}]^{\top}$, where each $\alpha_{i}>0$ and each $f_{i}$
is positive-definite and $C^{2}$. For this class of kernels $\+H_{z}$
is diagonal with $\+H_{z}^{ii}=\alpha_{i}\left.\frac{\partial^{2}f_{i}}{\partial x^{(i)}\partial y^{(i)}}\right|_{(z^{(i)},z^{(i)})}$ where $z^{(i)}$
is the $i$th element of the vector $z$. The positive-definite part of Assumption~\ref{assu:c2_etc} thus holds if these diagonal elements are strictly positive for all $z\in\mathcal{Z}$. 

Let us introduce the vector of monotonically increasing transformations: \begin{equation}\quad\psi(z)\coloneqq[\psi_1(z^{(1)})\,\cdots\,\psi_d(z^{(d)})]^{\top},\qquad\; \psi_i(z^{(i)})\coloneqq\alpha_i^{1/2} \int_{c_i^-}^{z^{(i)}}\left.\frac{\partial^{2}f_i}{\partial x^{(i)} \partial y^{(i)}} \right|_{(\xi,\xi)}^{1/2}\mathrm{d}\xi,\label{eq:psi}
\end{equation}
and write the vector of corresponding inverse transformations $\psi^{-1}$. For a generic path $\gamma$ in $\mathcal{M}$ with end-points $a,b$ and as usual $\eta_t\coloneqq \phi^{-1}(\gamma_t)$, define $\zeta_t\coloneqq \psi(\eta_t)$. Then:
\begin{align}
l(\gamma)=\int_{0}^{1}\left\langle \dot{\eta}_{t},\+H_{\eta_{t}}\dot{\eta}_{t}\right\rangle ^{1/2}\mathrm{d}t& = \int_0^1\|\dot{\zeta}_t\|\mathrm{d}t\label{eq:additive_l_gam}\\
&\geq  \left\|\int_0^1\dot{\zeta}_t\,\mathrm{d}t\right\|
=\|\zeta_1-\zeta_0\|=\|\psi\circ\phi^{-1}(b)-\psi\circ\phi^{-1}(a)\|,\label{eq:additive_path_length_lower_bound}
\end{align}
where the second and last equalities hold due to the definition of $\zeta_t$, and $\circ$ denotes elementwise composition. The quantity $\|\psi\circ\phi^{-1}(b)-\psi\circ\phi^{-1}(a)\|$ is thus a lower bound on path-length $l(\gamma)$ for any path $\gamma$ in $\mathcal{M}$ with end-points $a,b$. To see that there exists a path whose length achieves this lower bound, and hence that it is the geodesic distance in $\mathcal{M}$ between $a,b$, define $\tilde{\gamma}_t\coloneqq \phi\circ\psi^{-1}(\tilde{\zeta}_t)$ where $\tilde{\zeta}_{t}\coloneqq\psi\circ \phi^{-1}(a) + t[\psi\circ\phi^{-1}(b)-\psi\circ\phi^{-1}(a)]$. Differentiating $\tilde{\zeta}_{t}$ w.r.t. $t$ and substituting into  \eqref{eq:additive_l_gam} yields $l(\tilde{\gamma})=\|\psi\circ\phi^{-1}(b)-\psi\circ\phi^{-1}(a)\|$, as required.  Our conclusions in the case that $a=\phi(Z_i)$ and $b=\phi(Z_j)$ are summarised by:
\begin{prop}\label{prop:additive_kernels}
If $\mathcal{Z}$ is the Cartesian product of $d$ intervals $\mathcal{Z}_i\subset{\mathbb{R}}$, $i=1,\ldots,d$, and $f(x,y)=\sum_{i=1}^{d}\alpha_{i}f_{i}(x^{(i)},y^{(i)})$ where for each $i$, $\alpha_i>0$, $f_i$ is positive definite and $C^2$, and
$ \left.\frac{\partial^{2}f_{i}}{\partial x^{(i)}\partial y^{(i)}}\right|_{(z,z)}>0$ for all $z\in\mathcal{Z}_i$, then $\+D_\mathcal{M}^{ij}$ is equal to the Euclidean distance between $Z_i$ and $Z_j$ up to their coordinatewise-monotone transformation by $\psi$. If $ \left.\alpha_i\frac{\partial^{2}f_{i}}{\partial x^{(i)}\partial y^{(i)}}\right|_{(z,z)}$ is a positive constant over all $z\in\mathcal{Z}_i$, and over all $i=1,\ldots,d$, then $\+D_\mathcal{M}^{ij}\propto \+D_\mathcal{Z}^{ij}=\|Z_i-Z_j\|$.
\end{prop}

\section{Experiments} \label{sec:experiments}

\begin{figure}
\centering
  \includegraphics[width=\linewidth]{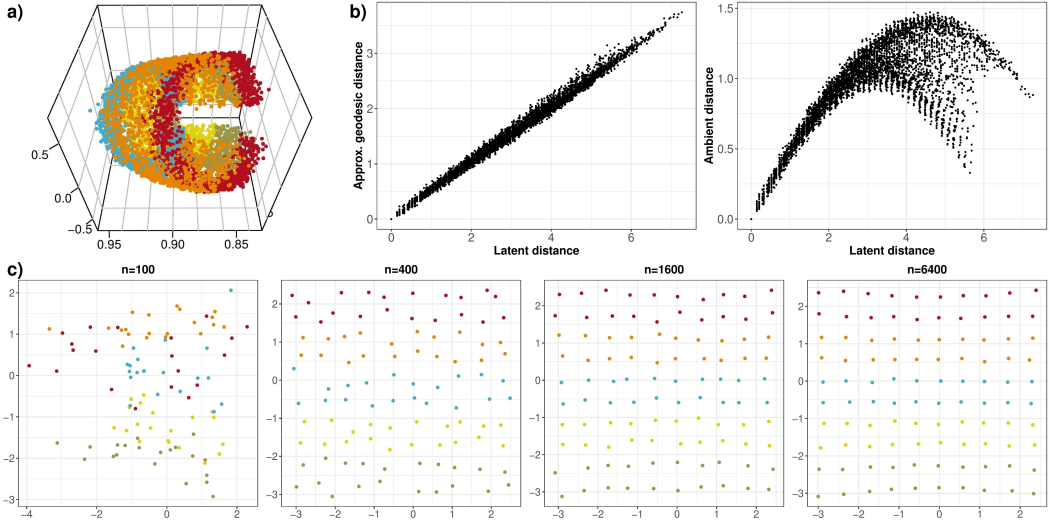}
  \caption{Simulated data example. a) Spectral embedding (first 3 dimensions), b) comparison of latent, approximate geodesic, and ambient distance and c) latent position recovery in the dense regime by spectral embedding followed by Isomap for increasing $n$. To aid visualisation, all plots in c) display a subset of $100$ estimated positions corresponding to true positions on a sub-grid which is common across $n$. Estimated positions are coloured according to their true $y$-coordinate.}
  \label{fig:dense_plots}
  \end{figure}

\subsection{Simulated data: latent position network model}\label{sec:simulation}
This section shows the theory at work in a pedagogical example, where $\+A$ is an adjacency matrix. Consider an undirected random graph following a latent position model with kernel $f(x,y) = \rho_n \{\cos(x^{(1)} - y^{(1)}) + \cos(x^{(2)} - y^{(2)}) +2\}/4$, operating on $\R^2 \times \R^2$. Here, the sequence $\rho_n$ is a sparsity factor which will either be constant, $\rho_n = 1$, or shrinking to zero sufficiently slowly, reflecting dense (degree grows linearly) and sparse (degree grows sublinearly) regimes respectively. The kernel is clearly translation-invariant, satisfies Assumptions \ref{assu:injective} and \ref{assu:c2_etc} and has finite rank, $p = 5$. The true latent positions $Z_i \in \mathbb{R}^2$ are deterministic and equally spaced points on a grid over the region $\mathcal{Z} = [-\pi + 0.25, \pi - 0.25] \times [-\pi + 0.25, \pi - 0.25]$, this range chosen to give valid probabilities and an interesting bottleneck in the 2-dimensional manifold $\mathcal{M}$. From Proposition~\ref{prop:translation_invariant}, the geodesic distance between $\phi(Z_i)$ and $\phi(Z_j)$ on $\mathcal{M}$ is equal to the Euclidean distance between $Z_i$ and $Z_j$, up to scaling, specifically $\+D_\mathcal{M}^{ij} = \rho_n\+D^{ij}_\mathcal{Z}/2$. 

Focusing first on the dense regime, for each $n = 100, 400, 1600, 6400$, we simulate a graph, and its spectral embedding $\hat X_1, \ldots, \hat X_n$ into $p=5$ dimensions. The first three dimensions of this embedding are shown in Figure~\ref{fig:dense_plots}a), and we see that the points gather about a two-dimensional, curved, manifold. To approximate geodesic distance, we compute a neighbourhood graph of radius $\epsilon$ from $\hat X_1, \ldots, \hat X_n$, choosing $\epsilon$ as small as possible subject to maintaining a connected graph. Figure~\ref{fig:dense_plots}b) shows that approximated geodesic distances roughly agree with the true Euclidean distances in latent space $\mathcal{Z}$ (up to the predicted scaling of 1/2), whereas there is significant distortion between ambient and latent distance ($\lVert \hat X_i - \hat X_j \rVert$ versus $\lVert Z_i - Z_j \rVert$). Finally, we recover the estimated latent positions in $\R^2$ using Isomap, which we align with the original by Procrustes (orthogonal transformation with scaling), in Figure~\ref{fig:dense_plots}c). As the theory predicts, the recovery error vanishes as $n$ increases.

In the appendix, we perform the same experiment in a sparse regime, $n\rho_{n} =\omega\{\log^{4} n\}$, chosen to ensure the spectral embedding is still consistent, and the recovery error still shrinks, but more slowly. We also implement other related approaches: UMAP, t-SNE applied to spectral embedding, node2vec directly into two dimensions and node2vec in five dimensions followed by Isomap (Figure~\ref{fig:method_comparison}). We use default configurations for all other hyperparameters. Together, the results support the central recommendation of this paper: to use matrix factorisation (e.g. spectral embedding, node2vec), followed by nonlinear dimension reduction (e.g. Isomap, UMAP, t-SNE).

\subsection{Real data: the global flight network} \label{sec:graph_example}
For this example, we use a publically available, clean version \citep{strohmeier2021crowdsourced} (with license details therein) of the crowdsourced OpenSky dataset. Six undirected graphs are extracted, for the months November 2019 to April 2020, connecting any two airports in the world for which at least a single flight was recorded ($n \approx 10,000$ airports in each). For each graph the adjacency matrix $\+A$ is spectrally embedded into $p = 10$ dimensions, degree-corrected by spherical projection \citep{lyzinski2014perfect,lei2015consistency,passino2020spectral}, after which we apply Isomap with $\epsilon$ chosen as the 5\% distance quantile. 

The dimension $p$ was chosen as a loose upper bound on the dimension estimate returned by the method of \citep{zhu2006automatic} on any of the individual graphs, as to facilitate comparison we prefer to use the same dimension throughout (e.g. to avoid artificial differences in variance) and so follow the recommendation of \cite{athreya2017statistical} to err on the large side (experiments with different choices of $p$ are in the appendix). In the degree-correction step we aim to remove network effects related to ``airport popularity'' (e.g. population, economic and airport size), after which geographic distance might be expected to be the principal factor deciding whether there should be a flight between two airports. Therefore, after applying Isomap, we might hope that $\hat Z_i$ could estimate the geographical location of airport $i$. The experiments take a few hours on a standard desktop (and the same is a loose upper bound on compute time for the other experiments in the paper).

The embedding for January is shown in Figures~\ref{fig:Jan_isomap})a-b), before and after Isomap. Both recover real geographic features, notably the relative positioning of London, Paris, Madrid, Lisbon, Rome. However, the embedding of the US is warped, and only after using Isomap do we see New York and Los Angeles correctly positioned at opposite extremes of the country. As further confirmation, Figures~\ref{fig:Jan_isomap})c-d) show the results restricted to North America, with the points coloured by latitude and longitude. For this continent, the accuracy of longitude recovery is particularly striking.

In line with the recommendation to overestimate $p$ \cite{athreya2017statistical}, one obtains a similar figure using $p=20$ before applying Isomap (Figure \ref{fig:n2v20}, appendix), whereas the corresponding figure with $p=2$ is too poor to show. 

In the appendix, the embeddings of all six months are shown, aligned using two-dimensional Procrustes, showing an important structural change in April. A statistical analysis of inter-continental geodesic distances suggests the change reflects severe flight restrictions in South America and Africa, at the beginning COVID-19 pandemic (the 11th March 2020 according to the WHO).

\begin{figure}[h!]
  \centering
  \includegraphics[width=.9\textwidth]{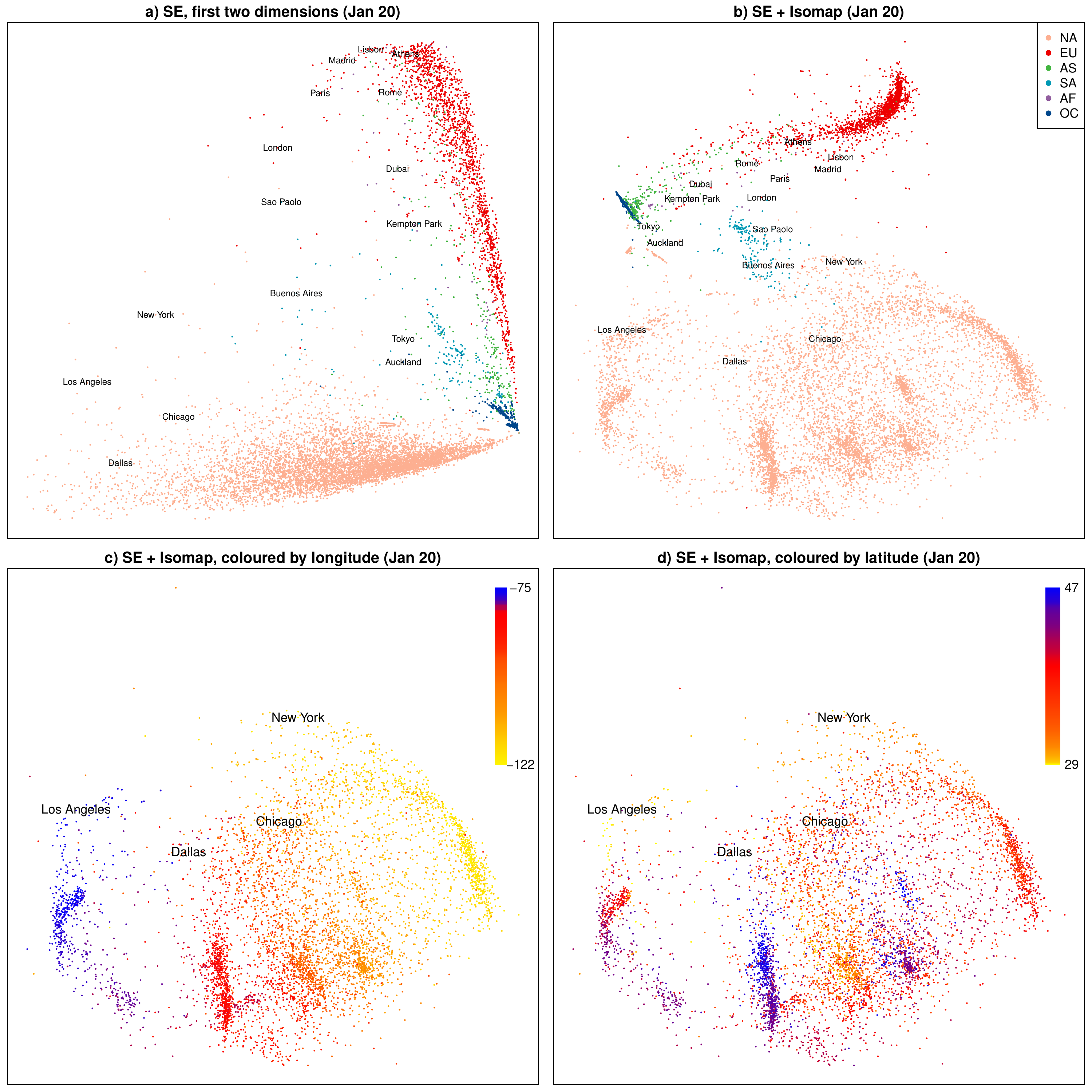}
  \caption{Visualisation of the global flight network, January 2020. a) Spectral embedding, b) spectral embedding followed by Isomap, colours indicating continents (NA = North America, EU = Europe, AS = Asia, AF = Africa, OC = Oceania), c-d) the same restricted to North America, showing only the airports within the continent's 5\%-95\% inter-quantile range of longitude and latitude, colours indicating the longitude (c) and latitude (d) of the airports.}\label{fig:Jan_isomap}
\end{figure}

\subsection{Real data: correlations between daily temperatures around the globe} \label{sec:temp_example}
In this example $\+A$ is a correlation matrix and we demonstrate a simple model-checking diagnostic informed by our theory. The raw data consist of average temperatures over each day, for several thousand days, recorded in cities around the globe. These data are open source, originate from Berkeley Earth \citep{berkeleyearth} and the particular data set analyzed is from \citep{kaggletempdata}. We used open source latitude and longitude data from \citep{worldcities}. See those references for license details.  Removing cities and dates for which data are missing yields a temperature time-series of $1450$ common days for each of the $n=2211$ cities. $\+A$ is the matrix of Pearson correlation coefficients between the $n$ time-series.

Figure~\ref{fig:temperature}b) shows the spectral embedding of $\+A$, with $p=2$ and points coloured by the latitude of the corresponding cities. Two visual features are striking: the concentration of points around a curved manifold, and a correspondence between latitude and location on the manifold. Figure~\ref{fig:temperature}c) shows  latitude against estimated latent positions from Isomap using the $k$-nearest neighbour graph with $k=200$, which is roughly $10\%$ of $n$, and with $d=1$. A clear monotone relationship appears. Our theoretical results can explain this phenomenon. Notice that when $d=1$, the additive structure of the kernel in Proposition~\ref{prop:additive_kernels} disappears. Hence Proposition~\ref{prop:additive_kernels} shows that in the case $d=1$, for \emph{any} kernel (meeting the basic requirements of section \ref{sec:setup}, including Assumptions \ref{assu:injective} and \ref{assu:c2_etc} of course) $\+D_{\mathcal{M}}^{ij}$ is equal to Euclidean distance between $Z_i$ and $Z_j$ up to their monotone transformation  by $\psi$ defined in \eqref{eq:psi}. In turn, this implies that if $\+A$ did actually follow the  model \eqref{eq:model} for some $f$ and some ``true'' latent positions $Z_1,\ldots,Z_n$ to which we had access, we should observe a monotone relationship between those true latent positions and the estimated positions $\hat{Z}_1,\ldots,\hat{Z}_n$ from Isomap.

\begin{figure}
  \centering
  \includegraphics[trim=0cm 1cm 0cm 0cm,clip,width=\textwidth]{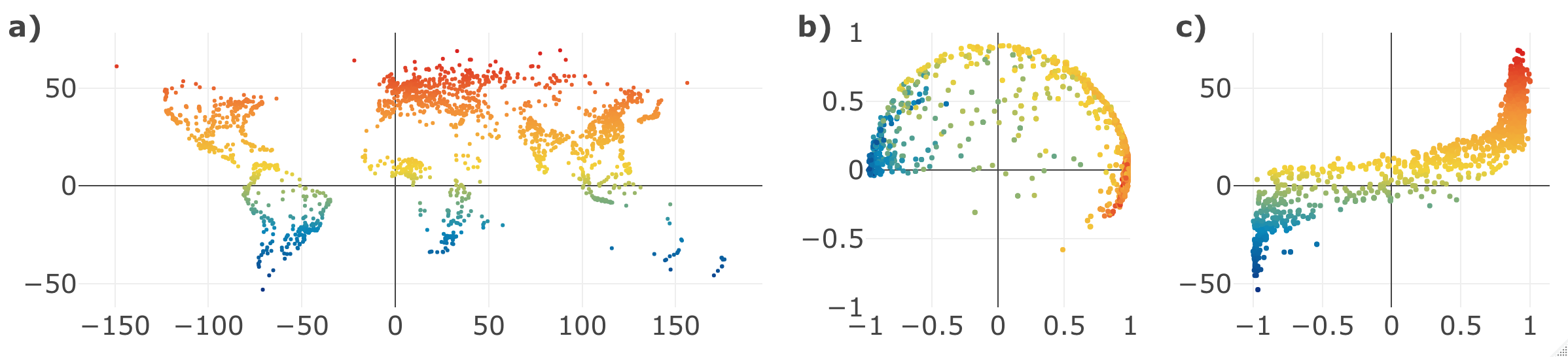}
  \caption{Temperature correlation example. a) Locations in degrees latitude and longitude of cities where temperatures were recorded, b) spectral embedding with $p=2$,  c) latitude of each city (vertical axis) against estimated latent position (horizontal axis) when $d=1$. In all plots, points are coloured by latitude of the corresponding cities.}\label{fig:temperature}
\end{figure}

The empirical monotonicity in Figure~\ref{fig:temperature}c) thus can be interpreted as indicating latitudes of the cities are plausible as ``true'' latent positions underlying the correlation matrix $\+A$, without us having to specify anything in practice about $f$. In further analysis (appendix) no such empirical monotone relationship is found between longitude and estimated latent position; this indicates longitude does $\emph{not}$ influence the correlations captured in $\+A$. A possible explanation is the daily averaging of temperatures in the underlying data: correlations between these average temperatures may be insensitive to longitude due to the rotation of the earth.

\section{Conclusion} \label{sec:conclusion}
Our research shows how matrix-factorisation and dimension-reduction can work together to reveal the true positions of objects based only on pairwise similarities such as connections or correlations. For the sake of exposition and reproducibility, we have used public datasets which can be interpreted without specialist domain knowledge, but the methods are potentially relevant to any scientific application involving large similarity matrices. Thinking about societal impact, our results highlight the depth of information in principle available about individuals, given network metadata, and we hope to raise awareness of these potential issues of privacy.

Concerning the limitations of the methods we have discussed, the bound in (\ref{eq:consistency_simple}) indicates that $n$ should be ``large'' in order for 
 $\+Q \hat{X}_i$ to approximate $\phi_p(Z_i)$ well. $n$  also has an impact on the performance of Isomap: heuristically one needs a high density and a large number of points on or near the manifold to get a good estimate of the geodesic distance. Thus the methods we propose are likely to perform poorly when $n$ is small, corresponding e.g. to a graph with a small number of vertices.
 
 On the other hand, in applications involving large networks or matrices (e.g. cyber-security, recommender systems), the data encountered are often sparse. This is good news for computational feasibility but bad news for statistical accuracy. In particular, for a graph, (\ref{eq:consistency_simple}) can only be expected to hold under logarithmically growing degree, the information-theoretic limit below which no algorithm can obtain asymptotically exact embeddings \citep{abbe2017community}. What this means in practice (as illustrated in our numerical results for the simulated data example in the appendix) is that the manifold may be very hard to distinguish, even for a large graph. Missing data may have a similarly negative impact on discerning manifold structure.  There could be substantial estimation gains in better integrating the factorisation and manifold estimation steps to overcome these difficulties.
 
 Another limitation is that our theory restricts attention to the case of positive-definite kernels. When $\+A$ is say a correlation or covariance matrix the positive-definite assumption on the kernel is of course natural, but when $\+A$ is say an adjacency matrix, it is a less natural assumption. The implications of removing the positive-definite assumption are the subject of ongoing research.

 \section*{Funding Transparency Statement} Nick Whiteley and Patrick Rubin-Delanchy's research was supported by Turing Fellowships from the Alan Turing Institute. Annie Gray's research was supported by a studentship from Compass, the EPSRC Centre for Doctoral Training in Computational Statistics and Data Science.
  
\bibliographystyle{plainnat}
\bibliography{references}

%%%%%%%%%%%%%%%%%%%%%%%%%%%%%%%%%%%%%%%%%%%%%%%%%%%%%%%%%%%%
\section*{Checklist}

% %%% BEGIN INSTRUCTIONS %%%
% The checklist follows the references.  Please
% read the checklist guidelines carefully for information on how to answer these
% questions.  For each question, change the default \answerTODO{} to \answerYes{},
% \answerNo{}, or \answerNA{}.  You are strongly encouraged to include a {\bf
% justification to your answer}, either by referencing the appropriate section of
% your paper or providing a brief inline description.  For example:
% \begin{itemize}
%   \item Did you include the license to the code and datasets? \answerYes{See Section~x.}
%   \item Did you include the license to the code and datasets? \answerNo{The code and the data are proprietary.}
%   \item Did you include the license to the code and datasets? \answerNA{}
% \end{itemize}
% Please do not modify the questions and only use the provided macros for your
% answers.  Note that the Checklist section does not count towards the page
% limit.  In your paper, please delete this instructions block and only keep the
% Checklist section heading above along with the questions/answers below.
% %%% END INSTRUCTIONS %%%

\begin{enumerate}

\item For all authors...
\begin{enumerate}
  \item Do the main claims made in the abstract and introduction accurately reflect the paper's contributions and scope?
    \answerYes{See Abstract and Section \ref{sec:introduction}.}
  \item Did you describe the limitations of your work?
    \answerYes{See Sections \ref{sec:spec_embed} and \ref{sec:setup}.}
  \item Did you discuss any potential negative societal impacts of your work?
    \answerYes{See Section \ref{sec:conclusion}.}
  \item Have you read the ethics review guidelines and ensured that your paper conforms to them?
    \answerYes{}
\end{enumerate}

\item If you are including theoretical results...
\begin{enumerate}
  \item Did you state the full set of assumptions of all theoretical results?
    \answerYes{See Section \ref{sec:theory} and Appendix \ref{appendix:theory}.}
	\item Did you include complete proofs of all theoretical results?
    \answerYes{See Section \ref{sec:shortest_paths} and Appendix \ref{appendix:theory}.}
\end{enumerate}

\item If you ran experiments...
\begin{enumerate}
  \item Did you include the code, data, and instructions needed to reproduce the main experimental results (either in the supplemental material or as a URL)?
    \answerYes{See Section \ref{sec:experiments} and \url{https://github.com/anniegray52/graphs}.}
  \item Did you specify all the training details (e.g., data splits, hyperparameters, how they were chosen)?
    \answerYes{See Section \ref{sec:experiments}.}
	\item Did you report error bars (e.g., with respect to the random seed after running experiments multiple times)?
    \answerYes{For Section \ref{sec:simulation} `set.seed(123)' has been used and supplementary figures for Section \ref{sec:graph_example} in the Appendix report error bars.}
	\item Did you include the total amount of compute and the type of resources used (e.g., type of GPUs, internal cluster, or cloud provider)?
    \answerYes{See Section \ref{sec:graph_example}.}
\end{enumerate}

\item If you are using existing assets (e.g., code, data, models) or curating/releasing new assets...
\begin{enumerate}
  \item If your work uses existing assets, did you cite the creators?
    \answerYes{See Section \ref{sec:experiments} for data and we list all R packages used in Appendix \ref{appendix:experiments}.}
  \item Did you mention the license of the assets?
    \answerYes{For details on licenses of data see Sections \ref{sec:graph_example}, \ref{sec:temp_example}, and for R packages see Appendix \ref{appendix:experiments}.}
  \item Did you include any new assets either in the supplemental material or as a URL? \answerYes{\url{https://github.com/anniegray52/graphs}}
  \item Did you discuss whether and how consent was obtained from people whose data you're using/curating?
    \answerYes{All datasets used are public datasets.}
  \item Did you discuss whether the data you are using/curating contains personally identifiable information or offensive content?
    \answerYes{}
\end{enumerate}

\item If you used crowdsourcing or conducted research with human subjects...
\begin{enumerate}
  \item Did you include the full text of instructions given to participants and screenshots, if applicable?
    \answerNA{}{}
  \item Did you describe any potential participant risks, with links to Institutional Review Board (IRB) approvals, if applicable?
    \answerNA{}
  \item Did you include the estimated hourly wage paid to participants and the total amount spent on participant compensation?
    \answerNA{}
\end{enumerate}

\end{enumerate}
\newpage
\appendix

\section{Theory} \label{appendix:theory}

\subsection{Supporting results and proof of Theorem \ref{thm:generic}} 
\begin{lem}
\label{lem:injective}If Assumption~\ref{assu:injective} holds, $\phi$
is injective.
\end{lem}

\begin{proof}
We prove the contrapositive. So suppose $\phi$ is not injective.
Then there must exist $x,y$, with $x\neq y$, such that $\phi(x)=\phi(y)$,
and hence for any $a$, $f(x,a)=\left\langle \phi(x),\phi(a)\right\rangle _{2}=\left\langle \phi(y),\phi(a)\right\rangle _{2}=f(y,a)$.
\end{proof}
The inverse of $\phi$ on $\mathcal{M}$ is denoted $\phi^{-1}$.
Let $\widetilde{\+H}_{z,z^{\prime}}\in\mathbb{R}^{d\times d}$ be the matrix with elements
\[
\widetilde{\+H}_{z,z^{\prime}}^{ij}\coloneqq\left.\dfrac{\partial^{2}f}{\partial x^{(i)}\partial y^{(i)}}\right|_{(z,z^{\prime})}.
\]

\begin{lem}
\label{lem:MVT}If Assumption~\ref{assu:c2_etc} holds, then for any
$x,y\in\mathcal{Z}$, there exists $z$ on the line segment with end-points
$x,y$ such that 
\[
\|\phi(x)-\phi(y)\|_{2}^{2}=\left\langle x-y,\left[\int_{0}^{1}\widetilde{\+H}_{z,y+s(x-y)}\mathrm{d}s\right](x-y)\right\rangle ,
\]
where the integral is element-wise.
\end{lem}

\begin{proof}
Fix any $x,y$ in $\mathcal{Z}$. Observe from (\ref{eq:mercer})
and the definition of $\phi$ that for any $x,y\in\mathcal{Z}$,
\[
\|\phi(x)-\phi(y)\|_{2}^{2}=f(x,x)+f(y,y)-2f(x,y).
\]
Now define

\[
g(u)\coloneqq f(u,x)-f(u,y),
\]
and so since $f$ is symmetric,
\[
g(x)-g(y)=\|\phi(x)-\phi(y)\|_{2}^{2}.
\]
By the mean value theorem, there exists $z$ on the line segment with
end-points $x,y$ (i.e. $z\in\widetilde{\mathcal{Z}}$) such that
\begin{align*}
g(x)-g(y) & =\left\langle \nabla g(z),x-y\right\rangle \\
 & =\left\langle \nabla_{x}f(z,y)-\nabla_{x}f(z,x),x-y\right\rangle 
\end{align*}
where $\nabla g$ is the gradient of $u\mapsto g(u)$ (with $x,y$
still considered fixed) and $\nabla_{x}f(z,u)$ is the gradient of
$x\mapsto f(x,u)$ evaluated at $z$ (with $u$ considered fixed).
Now considering the vector-valued mapping $u\mapsto\nabla_{x}f(z,u)$
with $z$ fixed, we have 
\[
\nabla_{x}f(z,x)-\nabla_{x}f(z,y)=\left[\int_{0}^{1}\widetilde{\+H}_{z,y+s(x-y)}\mathrm{d}s\right](x-y).
\]
Combining the above equalities gives:
\begin{align*}
\|\phi(x)-\phi(y)\|_{2}^{2} & =\left\langle \nabla_{x}f(z,y)-\nabla_{x}f(z,x),x-y\right\rangle \\
 & =\left\langle x-y,\left[\int_{0}^{1}\widetilde{\+H}_{z,y+s(x-y)}\mathrm{d}s\right](x-y)\right\rangle .
\end{align*}
\end{proof}
\begin{lem}
\label{lem:quad_form_bound}For any matrix $\+B \in\mathbb{R}^{d\times d}$ and $z\in\mathbb{R}^{d}$
, $\left|\left\langle z,\+Bz\right\rangle \right|^{1/2}\leq\|\+B\|_{\mathrm{F}}^{1/2}\|z\|$,
where $\|\cdot\|_{\mathrm{F}}$ is the Frobenius norm.
\end{lem}

\begin{proof}
\[
\left\langle z,\+B z\right\rangle =\frac{1}{2}\left\langle z,(\+B+\+B^{\top})z\right\rangle \leq\|z\|^{2}\lambda_{\mathrm{max}}\leq\|z\|^{2}\frac{1}{2}\|\+B+\+B^{\top}\|_{\mathrm{F}}\leq\|z\|^{2}\|\+B \|_{\mathrm{F}},
\]
where $\lambda_{\mathrm{max}}$ is the maximum eigenvalue of the symmetric
matrix $(\+B+\+B^{\top})/2$. Replacing $\+B$ by $-\+B$ and using $\|\+B\|_{\mathrm{F}}=\|-\+B\|_{\mathrm{F}}$
yields the lower bound $\left\langle z,\+B z\right\rangle \geq-\|z\|^{2}\|\+B \|_{\mathrm{F}}$.
\end{proof}
\begin{lem}
\label{lem:diff_h_bound}If Assumption~\ref{assu:c2_etc} holds, then
for any $\epsilon>0$ there exists $\delta>0$ such that for any $x,y\in\mathcal{Z}$
such that $\|x-y\|\leq\delta$ and any $\xi,z,z^{\prime}$ on the
line segment with endpoints $x,y$,
\[
\left|\left\langle x-y,(\+H_{\xi}-\widetilde{\+H}_{z,z^{\prime}})(x-y)\right\rangle \right|\leq\epsilon\|x-y\|^{2}.
\]
\end{lem}

\begin{proof}
For each $i,j$, since $(z,z^{\prime})\mapsto\widetilde{\+H}_{z,z^{\prime}}^{ij}$
is assumed continuous on $\widetilde{\mathcal{Z}}\times\widetilde{\mathcal{Z}}$,
and $\widetilde{\mathcal{Z}}$ is compact, by the Heine-Cantor theorem
$(z,z^{\prime})\mapsto\widetilde{\+H}_{z,z^{\prime}}^{ij}$ is in fact
uniformly continuous on $\widetilde{\mathcal{Z}}\times\widetilde{\mathcal{Z}}$.
Fix any $\epsilon>0$. Using this uniform continuity, there exists
$\delta>0$ such that for any $x,y\in\mathcal{Z}$, if $\|x-y\|\leq\delta$,
then for any $\xi,z,z^{\prime}$ on the line-segment with end-points
$x,y$, 
\[
\max_{i,j=1,\dots,d}\left|\widetilde{\+H}_{\xi,\xi}^{i,j}-\widetilde{\+H}_{z,z^{\prime}}^{i,j}\right|\leq\epsilon d^{-1},
\]
and so in turn 
\[
\left\Vert \widetilde{\+H}_{\xi,\xi}-\widetilde{\+H}_{z,z^{\prime}}\right\Vert _{\text{F}}\leq\epsilon,
\]
where $\left\Vert \cdot\right\Vert _{\text{F}}$ is the Frobenius
norm. Observing that $\widetilde{\+H}_{\xi,\xi}=\+H_{z}$, the result
then follows from Lemma \ref{lem:quad_form_bound}.
\end{proof}
\begin{prop}
\label{prop:lipschitz}If Assumptions~\ref{assu:injective} and \ref{assu:c2_etc}
hold, $\phi$ and $\phi^{-1}$ are each Lipschitz continuous with
respect to the norms $\|\cdot\|$ on $\mathcal{Z}$ and $\|\cdot\|_{2}$
on $\mathcal{M}$.
\end{prop}

\begin{proof}
As a preliminary note that for any $z\in\mathcal{Z}$, $\+H_{z}$ is
symmetric and positive-definite under Assumption~\ref{assu:c2_etc},
and let $\lambda_{z}^{\text{min}},\lambda_{z}^{\text{max}}$ be the
minimum and maximum eigenvalues of the matrix $\+H_{z}$. Since $\lambda_{z}^{\text{max}}=\|\+H_{z}\|_{\text{sp}}$,
the spectral norm of $\+H_{z}$, and the reverse triangle inequality
for this norm states $|\|\+H_{z}\|_{\text{sp}}-\|\+H_{z^{\prime}}\|_{\text{sp}}|\leq\|\+H_{z}-\+H_{z^{\prime}}\|_{\text{sp}}$,
the continuity in $z$ of the elements of $\+H_{z}$ under Assumption~\ref{assu:c2_etc} implies continuity of $z\mapsto\lambda_{z}^{\text{max}}$.
Similar consideration of $\lambda_{z}^{\text{min}}=\|\+H_{z}^{-1}\|_{\text{sp}}^{-1}$
together with 
\[
\|\+H_{z}^{-1}-\+H_{z^{\prime}}^{-1}\|_{\text{sp}}\leq\|\+H_{z}^{-1}\+H_{z^{\prime}}-\+I \|_{\text{sp}}\|\+H_{z^{\prime}}^{-1}\|_{\text{sp}}\leq\|\+H_{z}^{-1}\|_{\text{sp}}\|\+H_{z^{\prime}}-\+H_{z}\|_{\text{sp}}\|\+H_{z^{\prime}}^{-1}\|_{\text{sp}}
\]
shows that $z\mapsto\lambda_{z}^{\text{min}}$ is continuous. Due
to the compactness of $\mathcal{Z}$, we therefore find that $\lambda^{+}\coloneqq\sup_{z\in\mathcal{Z}}\lambda_{z}^{\text{max}}<\infty$,
and $\lambda^{-}\coloneqq\inf_{z\in\mathcal{Z}}\lambda_{z}^{\text{min}}>0$.

Our next objective in the proof of the proposition is to establish the Lipschitz continuity of $\phi$.
As a first step towards this, note that it follows from the identity
\[
\|\phi(x)-\phi(y)\|_{2}^{2}=f(x,x)+f(y,y)-2f(x,y),
\]
that the continuity of $(x,y)\mapsto f(x,y)$ implies continuity in
$\ell_{2}$ of $x\mapsto\phi(x)$. Now fix any $\epsilon_{1}>0$ and
consider any $x,y\in\mathcal{Z}$. By combining Lemmas \ref{lem:MVT}
and \ref{lem:diff_h_bound} there exists $\delta_{1}>0$ such that
if $\|x-y\|\leq\delta_{1}$, there exists $z$ on the line segment
with end-points $x,y$ such that:
\begin{align}
\|\phi(x)-\phi(y)\|_{2}^{2} & =\left\langle x-y,\left[\int_{0}^{1}\widetilde{\+H}_{z,y+s(x-y)}\mathrm{d}t\right](x-y)\right\rangle \nonumber \\
 & =\left\langle x-y,\+H_{z}(x-y)\right\rangle +\int_{0}^{1}\left\langle x-y,\left[\widetilde{\+H}_{z,y+s(x-y)}-\+H_{z}\right](x-y)\right\rangle \mathrm{d}s\label{eq:sum_of_quad_forms}\\
 & \leq(\lambda^{+}+\epsilon_{1})\|x-y\|^{2}.\label{eq:phi_lip_1}
\end{align}
On the other hand if $\|x-y\|>\delta_{1}$, 
\begin{equation}
\frac{\|\phi(x)-\phi(y)\|_{2}}{\|x-y\|}\leq c_{1}\delta_{1}^{-1},\label{eq:phi_lip_2}
\end{equation}
where $c_{1}\coloneqq\sup_{x,y\in\mathcal{Z}}\|\phi(x)-\phi(y)\|_{2}$
is finite since $\mathcal{Z}$ is compact and $\phi$ has already
been proved to be continuous in $\ell_{2}$. Combining (\ref{eq:phi_lip_1})
and (\ref{eq:phi_lip_2}) we obtain
\[
\|\phi(x)-\phi(y)\|_{2}\leq\left[c_{1}\delta_{1}^{-1}\vee(\lambda^{+}+\epsilon_{1})^{1/2}\right]\|x-y\|,\qquad\forall x,y\in\mathcal{Z}.
\]

It remains to prove Lipschitz continuity of $\phi^{-1}$. Fix $\epsilon_{2}\in(0,\lambda^{-})$.
Since $\mathcal{Z}$ is compact and $\phi$ is continuous, $\phi^{-1}$
is continuous on $\mathcal{M}$ \cite[Prop. 13.26]{sutherland2009introduction},
and then also uniformly continuous by the Heine-Cantor Theorem since $\mathcal{M}$ is compact. Putting this uniform continuity of $\phi^{-1}$
together with Lemmas \ref{lem:MVT} and \ref{lem:diff_h_bound}, via
the identity (\ref{eq:sum_of_quad_forms}), there exists $\delta_{2}>0$
such that for any $a,b\in\mathcal{M}$, if $\|a-b\|_{2}\leq\delta_{2}$
then
\[
\|a-b\|_{2}^{2}\geq(\lambda^{-}-\epsilon_{2})\|\phi^{-1}(a)-\phi^{-1}(b)\|^{2}.
\]
 On the other hand, if $\|a-b\|_{2}>\delta_{2}$, 
\[
\frac{\|\phi^{-1}(a)-\phi^{-1}(b)\|}{\|a-b\|_{2}}\leq c_{2}\delta_{2}^{-1},
\]
where $c_{2}\coloneqq\sup_{x,y\in\mathcal{Z}}\|x-y\|$ is finite since
$\mathcal{Z}$ is compact. Therefore
\[
\|\phi^{-1}(a)-\phi^{-1}(b)\|\leq\left[c_{2}\delta_{2}^{-1}\vee(\lambda^{-}-\epsilon_{2})^{-1/2}\right]\|a-b\|_{2},\qquad\forall a,b\in\mathcal{M}.
\]
\end{proof}
\begin{lem}
\label{lem:finite length}If Assumptions \ref{assu:injective} and
\ref{assu:c2_etc} hold, then for any $a,b\in\mathcal{M}$ and a path
$\gamma$ in $\mathcal{M}$ with end-points $a,b$ such that $\ell(\gamma)<\infty$,
the mapping $\eta:[0,1]\to\mathcal{Z}$ defined by $\eta_{t}\coloneqq\phi^{-1}(\gamma_{t})$
is a path in $\mathcal{Z}$ with end-points $\phi^{-1}(a),\phi^{-1}(b)$,
and $l(\eta)<\infty$.
\end{lem}

\begin{proof}
By Proposition \ref{prop:lipschitz}, $\phi^{-1}$ is continuous,
which combined with the continuity of $t\mapsto\gamma_{t}$ implies
continuity of $t\mapsto\phi^{-1}(\gamma_{t})$, so $\eta$ is indeed a path
in $\mathcal{Z}$, and the end points of $\eta$ are clearly $\phi^{-1}(a),\phi^{-1}(b)$.
Proposition \ref{prop:lipschitz} establishes that moreover $\phi^{-1}$
is Lipschitz continuous, and then $l(\gamma)<\infty$ implies $l(\eta)<\infty$
due to the definition of path-length.
\end{proof}
\begin{lem}
\label{lem:ab_lower_bound}For any $a\geq0$ and $b$ such that $|b|\leq a$,
\[
|a|^{1/2}-|b|^{1/2}\leq(a+b)^{1/2}\leq|a|^{1/2}+|b|^{1/2}.
\]
\end{lem}

\begin{proof}
First prove the lower bound. For any $a,b$ as in the statement, let
$c=a-|b|$, so that $c\geq0$, and set $x=|b|^{1/2}$ and $y=c^{1/2}$.
Since $x,y\geq0$, application of the Euclidean triangle inequality
in $\mathbb{R}^{2}$ to the pair of vectors $[x\;0]^{\top}$, $[0\;y]^{\top}$
gives the fact: $(x^{2}+y^{2})^{1/2}\leq x+y$, hence $a^{1/2}=(|b|+c)^{1/2}\leq|b|^{1/2}+c^{1/2}=|b|^{1/2}+(a-|b|)^{1/2},$
or equivalently:
\begin{equation}
(a-|b|)^{1/2}\geq a^{1/2}-|b|^{1/2}.\label{eq:ab_lower_bound_1}
\end{equation}
By the reverse triangle inequality and the assumptions on $a$ and
$b$, 
\begin{equation}
(a+b)^{1/2}=|a+b|^{1/2}\geq(a-|b|)^{1/2}.\label{eq:ab_lower_bound_2}
\end{equation}
Combining (\ref{eq:ab_lower_bound_1}) and (\ref{eq:ab_lower_bound_2})
completes the proof of the lower bound in the statement.

For the upper-bound in the statement, let $c=a^{1/2}$ and $d=|b|^{1/2}$.
Then 
\[
a+b\leq c^{2}+d^{2}=(c+d)^{2}-2cd\leq(c+d)^{2},
\]
which implies
\[
(a+b)^{1/2}\leq a^{1/2}+|b|^{1/2}
\]
as required.
\end{proof}
\begin{proof}
[Proof of Theorem \ref{thm:generic}]By Lemma \ref{lem:injective},
$\phi$ is injective; by Proposition \ref{prop:lipschitz}, $\text{\ensuremath{\phi}}$
and its inverse on $\mathcal{M}$, namely $\phi^{-1}$, are Lipschitz;
and by Lemma \ref{lem:finite length}, for $\gamma$ and $\eta$ as
in the statement, $\eta$ is a path as claimed with $l(\eta)<\infty$.

For the remainder of the proof, fix any $\epsilon>0$. By the definition
of the path-length $l(\gamma)$, there exists a partition $\widetilde{\mathcal{P}}_{\epsilon}$
such that 
\begin{equation}
l(\gamma)-\epsilon/2\leq\chi(\gamma,\widetilde{\mathcal{P}}_{\epsilon})\leq l(\gamma).\label{eq:P_eps}
\end{equation}

Let $\mathcal{P}=(t_{0},\ldots,t_{n})$ be any partition, and fix
any $k$ such that $1\leq k\leq n$. By Lemma \ref{lem:MVT} there
exists $z$ on the line segment with end-points $\eta_{t_{k-1}},\eta_{t_{k}}$
such that 
\begin{align}
\|\gamma_{t_{k}}-\gamma_{t_{k-1}}\|_{2}^{2} & =\|\phi(\eta_{t_{k}})-\phi(\eta_{t_{k-1}})\|_{2}^{2}\nonumber\\
 & =\left\langle \eta_{t_{k}}-\eta_{t_{k-1}},\left[\int_{0}^{1}\widetilde{\+H}_{z,\eta_{t_{k-1}}+s(\eta_{t_{k}}-\eta_{t_{k-1}})}\mathrm{d}s\right](\eta_{t_{k}}-\eta_{t_{k-1}})\right\rangle\label{eq:a_k_plus_b_k} \\
 & =a_{k}+b_{k}
\end{align}
where
\begin{align*}
 & a_{k}\coloneqq\left\langle \eta_{t_{k}}-\eta_{t_{k-1}},\+H_{\eta_{t_{k-1}}}(\eta_{t_{k}}-\eta_{t_{k-1}})\right\rangle \\
 & b_{k}\coloneqq\int_{0}^{1}\left\langle \eta_{t_{k}}-\eta_{t_{k-1}},\left[\widetilde{\+H}_{z,\eta_{t_{k-1}}+s(\eta_{t_{k}}-\eta_{t_{k-1}})}-\+H_{\eta_{t_{k-1}}}\right](\eta_{t_{k}}-\eta_{t_{k-1}})\right\rangle \mathrm{d}s.
\end{align*}
Under Assumption \ref{assu:c2_etc}, $\+H_{\eta_{t_{k-1}}}$ is positive-definite,
so $a_{k}\geq0$, and by \eqref{eq:a_k_plus_b_k},
$a_k + b_k \geq 0$,  so we must have $|b_{k}|\leq a_{k}$. Lemma \ref{lem:ab_lower_bound} then
gives 
\[
a_{k}^{1/2}-|b_{k}|^{1/2}\leq(a_{k}+b_{k})^{1/2}\leq a_{k}^{1/2}+|b_{k}|^{1/2},
\]
 hence: 
\begin{equation}
\left|\|\gamma_{t_{k}}-\gamma_{t_{k-1}}\|_{2}-a_{k}^{1/2}\right|\leq|b_{k}|^{1/2}.\label{eq:gamma_minus_a_bound}
\end{equation}

By Lemma \ref{lem:diff_h_bound}, there exists $\delta>0$ such that
\begin{equation}
\max_{k=1,\ldots,n}\|\eta_{t_{k}}-\eta_{t_{k-1}}\|\leq\delta\quad\Rightarrow\quad|b_{k}|^{1/2}\leq\frac{\epsilon}{2}\frac{1}{l(\eta)}\|\eta_{t_{k}}-\eta_{t_{k-1}}\|,\quad\forall1\leq k\leq n.\label{eq:mesh_implication}
\end{equation}
Since $\eta$ is a path, it is continuous on the compact set $[0,1]$,
and then in fact uniformly continuous by the Heine-Cantor Theorem.
Hence there exists a suitably fine partition $\mathcal{P}_{\epsilon,\delta}\supseteq\widetilde{\mathcal{P}}_{\epsilon}$
such that if $\mathcal{P}=(t_{0},\ldots,t_{n})\supseteq\mathcal{P}_{\epsilon,\delta}$,
$\max_{k=1,\ldots,n}\|\eta_{t_{k}}-\eta_{t_{k-1}}\|\leq\delta$ and
in turn from (\ref{eq:mesh_implication}),
\begin{equation}
\sum_{k=1}^{n}|b_{k}|^{1/2}\leq\frac{\epsilon}{2}\frac{1}{l(\eta)}\sum_{k=1}^{n}\|\eta_{t_{k}}-\eta_{t_{k-1}}\|\leq\frac{\epsilon}{2},\label{eq:beta_bound}
\end{equation}
where the final inequality holds due to the definition of $l(\eta)$
as the length of $\eta$.

Combining (\ref{eq:gamma_minus_a_bound}) and (\ref{eq:beta_bound}),
if again $\mathcal{P}\supseteq\mathcal{P}_{\epsilon,\delta}$,
\begin{align}
\left|\chi(\gamma,\mathcal{P})-\sum_{k=1}^{n}a_{k}^{1/2}\right| & =\left|\sum_{k=1}^{n}\|\gamma_{t_{k}}-\gamma_{t_{k-1}}\|_{2}-a_{k}^{1/2}\right|\nonumber \\
 & \leq\sum_{k=1}^{n}|b_{k}|^{1/2}\nonumber \\
 & \leq\frac{\epsilon}{2}.\label{eq:S_minus_sum_a}
\end{align}
Recalling from (\ref{eq:P_eps}) the defining property of $\widetilde{\mathcal{P}}_{\epsilon}$
and using $\mathcal{P}\supseteq\mathcal{P}_{\epsilon,\delta}\supseteq\widetilde{\mathcal{P}}_{\epsilon}$,
the triangle inequality for the $\|\cdot\|_{2}$ norm gives
\[
l(\gamma)-\frac{\epsilon}{2}\leq\chi(\gamma,\widetilde{\mathcal{P}}_{\epsilon})\leq\chi(\gamma,\mathcal{P})\leq l(\gamma).
\]
Combined with (\ref{eq:S_minus_sum_a}), we finally obtain that if
$\mathcal{P}\supseteq\mathcal{P}_{\epsilon,\delta}$,
\[
l(\gamma)-\epsilon\leq\sum_{k=1}^{n}a_{k}^{1/2}\leq l(\gamma)+\frac{\epsilon}{2}
\]
and the proof of (\ref{eq:path_length}) is completed by taking $\mathcal{P}_{\epsilon}$
as appears in the statement to be $\mathcal{P}_{\epsilon,\delta}$.

The first equality in (\ref{eq:riemannian_path_length}) can be proved
by a standard argument - e.g., \cite[p.137]{rudin1976principles}.
The second inequality in (\ref{eq:riemannian_path_length}) is proved
by passing to the limit of the summation in (\ref{eq:path_length})
along any sequence of partitions $\mathcal{P}^{(m)}=(t_{0}^{(m)},\ldots,t_{n(m)}^{(m)})$,
$m\geq1$, with $\mathcal{P}^{(m)}\supseteq\mathcal{P}_{\epsilon,\delta}$
such that $\lim_{m\to\infty}\max_{k=1,\ldots,n(m)}|t_{k}^{(m)}-t_{k-1}^{(m)}|=0$.
\end{proof}

\subsection{Deriving geodesic distances from \eqref{eq:path_length} rather than from \eqref{eq:riemannian_path_length}}
Recall from Section~\ref{sec:shortest_paths} that the general strategy to derive the geodesic distance associated with each family of kernels (translation invariant, inner-product, additive)  is: 
\begin{enumerate}[(i)]
    \item identify a lower bound on $l(\gamma)$ which holds over all paths $\gamma$ in $\mathcal{M}$ which have generic end-points $a,b\in\mathcal{M}$ in common, then
    \item show there exists a path whose length is equal to this lower bound.
\end{enumerate}
In Section~\ref{sec:shortest_paths} this strategy was executed for each family of kernels starting from the expression for $l(\gamma)$ given in \eqref{eq:riemannian_path_length}. In the proofs of Lemmas \ref{lem:trans_inv}-\ref{lem:additive} below we show how step (i) is performed if we start not from \eqref{eq:riemannian_path_length} but rather from \eqref{eq:path_length}, the latter being more general because continuous differentiability of the paths is relaxed to continuity.  The key message of these three lemmas regarding step (i) is that we obtain exactly the same lower bounds on $l(\gamma)$ as are derived from \eqref{eq:riemannian_path_length} in Section~\ref{sec:shortest_paths}. The reader is directed to Section~\ref{sec:shortest_paths} for the details of how Assumption \ref{assu:c2_etc} is verified for each family of kernels; to avoid repetition we don't re-state all those details here.

\begin{lem}\label{lem:trans_inv}
Consider the family of translation invariant kernels described in Section~\ref{sec:shortest_paths} and let $\+G$ be as defined there. For any $a,b\in\mathcal{M}$ and any path $\gamma\in\mathcal{M}$ with end-points $a,b$,
\[
l(\gamma)\geq \|\+G^{\top}[\phi^{-1}(b)-\phi^{-1}(a)]\|.
\]
If we define $\tilde{\eta}$ to be the path in  $\mathcal{Z}$ given by
\[
\tilde{\eta}_{t}\coloneqq\phi^{-1}(a)+t[\phi^{-1}(b)-\phi^{-1}(a)],\quad t\in[0,1],
\]
then $\tilde{\gamma}$ defined by $\tilde{\gamma}_t \coloneqq \phi(\tilde{\eta}_t)$ is a path in $\mathcal{M}$ with end-points $a,b$ and $l(\tilde{\gamma})=\|\+G^{\top}[\phi^{-1}(b)-\phi^{-1}(a)]\|.$
\end{lem}
\begin{proof}
Applying Theorem~\ref{thm:generic}, fix any $\epsilon>0$ and let $\mathcal{P}_{\epsilon}$ be a partition such that for any partition $\mathcal{P}=(t_0,\ldots,t_n)$ satisfying $\mathcal{P}_{\epsilon}\subseteq \mathcal{P}$,
\begin{equation}
\left|l(\gamma)-\sum_{k=1}^{n}\left\langle \eta_{t_{k}}-\eta_{t_{k-1}},\+H_{\eta_{t_{k-1}}}(\eta_{t_{k}}-\eta_{t_{k-1}})\right\rangle ^{1/2}\right|\leq\epsilon.\label{eq:l_gam_appendix1}
\end{equation}
Recalling from Section~\ref{sec:shortest_paths} that for this family of translation invariant kernels $\+H_z = \+G \+G^{\top}$ for all $z\in\mathcal{Z}$,  the triangle inequality for the $\|\cdot\|$ norm combined with \eqref{eq:l_gam_appendix1} gives
\begin{align*}
l(\gamma) &\geq -\epsilon +\sum_{k=1}^{n}\left\langle \eta_{t_{k}}-\eta_{t_{k-1}},\+H_{\eta_{t_{k-1}}}(\eta_{t_{k}}-\eta_{t_{k-1}})\right\rangle ^{1/2} .\\
& = -\epsilon + \sum_{k=1}^{n}\|\+G^{\top}(\eta_{t_{k}}-\eta_{t_{k-1}})\|  \\
& \geq  -\epsilon + \left\lVert\sum_{k=1}^{n} \+G^{\top}(\eta_{t_{k}}-\eta_{t_{k-1}}) \right\rVert  \\
& =  -\epsilon + \left\lVert  \+G^{\top}(\eta_{1}-\eta_{0}) \right\rVert  \\
& = -\epsilon + \left\lVert  \+G^{\top}[\phi^{-1}(b)-\phi^{-1}(a)] \right\rVert . \\
\end{align*}
The proof of the lower bound in the statement is then complete since $\epsilon$ was arbitrary. To complete the proof of the lemma, observe that from the definition of $\tilde{\eta}$ in the statement,
\begin{align*}
&\sum_{k=1}^{n}\left\langle \tilde{\eta}_{t_{k}}-\tilde{\eta}_{t_{k-1}},\+H_{\tilde{\eta}_{t_{k-1}}}(\tilde{\eta}_{t_{k}}-\tilde{\eta}_{t_{k-1}})\right\rangle ^{1/2}\\
&=\sum_{k=1}^{n} \|(t_k - t_{k-1})\+G^{\top}[\phi^{-1}(b)-\phi^{-1}(a)]\|\\
&=\|\+G^{\top}[\phi^{-1}(b)-\phi^{-1}(a)]\|\sum_{k=1}^{n} (t_k - t_{k-1}) \\
&= \|\+G^{\top}[\phi^{-1}(b)-\phi^{-1}(a)]\|,
\end{align*}
and the proof of the lemma is then complete, because  $\epsilon$ in \eqref{eq:l_gam_appendix1} being arbitrary implies $l(\tilde{\gamma})=\|\+G^{\top}[\phi^{-1}(b)-\phi^{-1}(a)]\|$.

\end{proof}

\begin{lem}\label{lem:inner_prod}
Consider the family of inner-product kernels of the form $f(x,y)=g(\langle x,y \rangle )$ as described in Section~\ref{sec:shortest_paths} where  $g^\prime(1)>0$ . For any $a,b\in\mathcal{M}$ and any path $\gamma\in\mathcal{M}$ with end-points $a,b$,
\begin{equation*}
l(\gamma) \geq g^{\prime}(1)^{1/2} \arccos\left\langle \phi^{-1}(a),\phi^{-1}(b)\right\rangle.     
\end{equation*}
If $\tilde{\eta}$ is a shortest circular arc in $\mathcal{Z}$ with end-points $\phi^{-1}(a),\phi^{-1}(b)$, then $\tilde{\gamma}$ defined by $\tilde{\gamma}_t\coloneqq \phi(\tilde{\eta}_t)$ satisfies $l(\tilde{\gamma})= g^{\prime}(1)^{1/2} \arccos\left\langle \phi^{-1}(a),\phi^{-1}(b)\right\rangle$.
\end{lem}
\begin{proof}
As usual, let $\eta$ be the path in $\mathcal{Z}$ defined by $\eta_t \coloneqq \phi^{-1}(\gamma_t)$. Then from the definition of path-length and the triangle inequality for the $\|\cdot\|$ norm, for any $\delta >0$, there exists a partition $\mathcal{P}^{\star}_\delta$ such that for any $\mathcal{P}^{\star}=(t_0^{\star},\ldots,t_n^{\star})$ satisfying $\mathcal{P}^{\star}_\delta\subseteq \mathcal{P}^{\star}$,
\begin{equation}
\sum_{k=1}^{n}
\|\eta_{t_{k}^{\star}}-\eta_{t_{k-1}^{\star}}\|  \geq l(\eta) -  \frac{\delta}{g^\prime(1)^{1/2}}. \label{eq:inne_prod_proof}   
\end{equation}

Fix any $\epsilon>0$ and let $\mathcal{P}_{\epsilon}$ be as in Theorem~\ref{thm:generic} and then take  $\mathcal{P}=(t_0,\ldots,t_n)$ to be defined by $\mathcal{P}=\mathcal{P}_{\epsilon} \cup \mathcal{P}^{\star}_\delta$, so by construction we have simultaneously  $\mathcal{P}_{\epsilon}\subseteq \mathcal{P}$ and $\mathcal{P}^{\star}_\delta\subseteq \mathcal{P}$. Then from Theorem~\ref{thm:generic},
\begin{equation}
\left|l(\gamma)-\sum_{k=1}^{n}\left\langle \eta_{t_{k}}-\eta_{t_{k-1}},\+H_{\eta_{t_{k-1}}}(\eta_{t_{k}}-\eta_{t_{k-1}})\right\rangle ^{1/2}\right|\leq\epsilon.\label{eq:inner_prod_thm1}
\end{equation}
Combined with the fact that for this family of kernels $\+H_z = g^\prime (1) \+I + g^{\prime\prime}(1) z z^{\top}$ where $g^\prime(1)>0$ and $g^{\prime\prime}(1) \geq 0$, we obtain
\begin{align*}
l(\gamma) &\geq -\epsilon+ \sum_{k=1}^{n}\left\langle \eta_{t_{k}}-\eta_{t_{k-1}},\+H_{\eta_{t_{k-1}}}(\eta_{t_{k}}-\eta_{t_{k-1}})\right\rangle ^{1/2}\\
& =  -\epsilon + \sum_{k=1}^{n}
\left(g^{\prime}(1)\|\eta_{t_{k}}-\eta_{t_{k-1}}\|^2 + g^{\prime\prime}(1) |\langle\eta_{t_{k}}-\eta_{t_{k-1}},z\rangle|^2\right)^{1/2} \\
& \geq -\epsilon + g^{\prime}(1)^{1/2} \sum_{k=1}^{n}
\|\eta_{t_{k}}-\eta_{t_{k-1}}\|\\
&\geq -\epsilon + g^{\prime}(1)^{1/2}l(\eta) - \delta,
\end{align*}
where the penultimate inequality uses $g^{\prime\prime}(1)\geq 0$ and the final inequality holds by taking $\mathcal{P}^{\star}$ in \eqref{eq:inne_prod_proof} to be $\mathcal{P}$. Since $\epsilon$ and $\delta$ were arbitrary, we have shown $l(\gamma)\geq g^{\prime}(1)^{1/2} l(\eta)$. Recall that here $\eta$ is a path in $\mathcal{Z}=\{\|x\|\in\mathbb{R}^d:\|x\|=1\}$ with end-points $\phi^{-1}(a),\phi^{-1}(b)$. Hence $l(\eta)$ is lower-bounded by the Euclidean geodesic distance in $\mathcal{Z}$ between $\phi^{-1}(a)$ and $\phi^{-1}(b)$, which is   $\arccos\left\langle \phi^{-1}(a),\phi^{-1}(b)\right\rangle$ because $\mathcal{Z}$ is a radius-$1$ sphere centered at the origin.

With $\tilde{\eta}$ and $\tilde{\gamma}$ as defined in the statement, taking $\eta$ in \eqref{eq:inner_prod_thm1} to be $\tilde{\eta}$, refining $\mathcal{P}$ and using $\langle\tilde{\eta}_t,\dot{\tilde{\eta}}_t\rangle=0$ (see discussion in Section \ref{sec:shortest_paths}) we find $l(\tilde{\gamma})=g^{\prime}(1)^{1/2}l(\tilde{\eta})$, where by definition of $\tilde{\eta}$, $l(\tilde{\eta})=\arccos\left\langle \phi^{-1}(a),\phi^{-1}(b)\right\rangle$.

\end{proof}

\begin{lem}\label{lem:additive}
Consider the family of additive kernels described in Section~\ref{sec:shortest_paths}. For any $a,b\in\mathcal{M}$ and any path $\gamma\in\mathcal{M}$ with end-points $a,b$,
$$
l(\gamma) \geq \|\psi\circ \phi^{-1}(b) - \psi \circ \phi^{-1}(a) \|.
$$
If we define  $\tilde{\zeta}_{t}\coloneqq\psi\circ \phi^{-1}(a) + t[\psi\circ\phi^{-1}(b)-\psi\circ\phi^{-1}(a)]$ and let $\tilde{\gamma}$ be defined by $\tilde{\gamma}_t\coloneqq \phi\circ \psi^{-1}(\zeta_t)$, then $l(\tilde{\gamma})=\|\psi\circ \phi^{-1}(b) - \psi \circ \phi^{-1}(a) \|$.
\end{lem}

\begin{proof}
 The compactness of $\mathcal{Z}$ and the continuity of $z\mapsto \left.\dfrac{\partial^2 f_i}{\partial x^{(i)} y^{(i)}}\right\rvert^{1/2}_{(z,z)}$ for each $i=1,\ldots,d$ implies the uniform-continuity of the latter by the Heine-Cantor theorem. Hence for any $\delta_1>0$, there exists $\delta_2>0$ such that for all  $i=1,\ldots,d$ and $z_1^{(i)},z_2 ^{(i)}\in\mathcal{Z}_i$,
\begin{equation}
|z_1^{(i)}-z_2^{(i)}|\leq \delta_2 \quad \Rightarrow \quad \alpha_i ^{1/2} \left\lvert \left.\dfrac{\partial^2 f_i}{\partial x^{(i)} y^{(i)}}\right\rvert^{1/2}_{(z_1^{(i)},z_1^{(i)})} - \left.\dfrac{\partial^2 f_i}{\partial x^{(i)} y^{(i)}}\right\rvert^{1/2}_{(z_2^{(i)},z_2^{(i)})}\right\rvert\leq \frac{\delta_1}{l(\eta)}.\label{eq:additive_cont}
\end{equation}
Fix any $\epsilon>0$ and let $\mathcal{P}_{\epsilon}$ be as in Theorem~\ref{thm:generic}. We now claim there exists a partition $\mathcal{P}=(t_0,\ldots, t_n)$ satisfying simultaneously $\mathcal{P}_{\epsilon}\subseteq \mathcal{P}$ and 
\begin{equation}
\max_{i=1,\ldots,d}\,\max_{k=1,\ldots,n}|\eta_{t_k}^{(i)} - \eta_{t_{k-1}}^{(i)}|\leq \delta_2.\label{eq:additive_P}
\end{equation}
To see that such a partition exists, note that with $\eta_t\coloneqq \phi^{-1}(\gamma_t)$, $t\mapsto \eta_t$ is continuous on the compact set $[0,1]$ hence uniformly continuous by the Heine-Cantor theorem. Thus for any  $s,t\in[0,1]$ sufficiently close to each other, $\|\eta_s - \eta_t\|$ can be made less than or equal to $\delta_2$, which implies $\max_{i=1,\ldots,d}|\eta_s^{(i)}-\eta_t^{(i)}|\leq \delta_2$. Thus starting from $\mathcal{P}_{\epsilon}$, if we subsequently add points to this partition until $\max_{k=1,\ldots,n}|t_{k}-t_{k-1}|$ is sufficiently small then we will arrive at a partition $\mathcal{P}$ with the required properties, as claimed.

Now with this partition $\mathcal{P}$ in hand, fix any $i=1,\ldots,d$ and $k=1,\ldots, n$. We then have
\begin{align}
\zeta_{t_k}^{(i)} - \zeta_{t_{k-1}}^{(i)} &=\psi_i(\eta_{t_k}^{(i)}) - \psi_i(\eta_{t_{k-1}}^{(i)})\label{eq:zeta_decomp}\\
&=\alpha_i^{1/2}\int_{\eta_{t_{k-1}}^{(i)}}^{\eta_{t_k}^{(i)}}  \left.\dfrac{\partial^2 f_i}{\partial x^{(i)} y^{(i)}}\right\rvert^{1/2}_{(\xi,\xi)} \mathrm{d}\xi\nonumber \\
&= \alpha_i^{1/2} \int_{\eta_{t_{k-1}}^{(i)}}^{\eta_{t_k}^{(i)}}\left[ \left.\dfrac{\partial^2 f_i}{\partial x^{(i)} y^{(i)}}\right\rvert^{1/2}_{(\xi,\xi)} - \left.\dfrac{\partial^2 f_i}{\partial x^{(i)} y^{(i)}}\right\rvert^{1/2}_{(\eta_{t_{k-1}}^{(i)},\eta_{t_{k-1}}^{(i)})}\right] \mathrm{d}\xi \nonumber\\
&\quad + \alpha_i^{1/2} \left.\dfrac{\partial^2 f_i}{\partial x^{(i)} y^{(i)}}\right\rvert^{1/2}_{(\eta_{t_{k-1}}^{(i)},\eta_{t_{k-1}}^{(i)})} (\eta_{t_k}^{(i)} - \eta_{t_{k-1}}^{(i)})\nonumber\\
&\leq  \left[\frac{\delta_1}{l(\eta)}+ \alpha_i^{1/2} \left.\dfrac{\partial^2 f_i}{\partial x^{(i)} y^{(i)}}\right\rvert^{1/2}_{(\eta_{t_{k-1}}^{(i)},\eta_{t_{k-1}}^{(i)})} \right] (\eta_{t_k}^{(i)} - \eta_{t_{k-1}}^{(i)}),\label{eq:zeta_bound}
\end{align}
where the upper bound is due to \eqref{eq:additive_cont}-\eqref{eq:additive_P}. Squaring both sides, summing over $i$ and then applying the triangle inequality gives
$$
\|\+H^{(1/2)}_{\eta_{t_{k-1}}}(\eta_{t_k}-\eta_{t_{k-1}})\|\geq \|\zeta_{t_k} - \zeta_{t_{k-1}}\|- \frac{\delta_1}{l(\eta)}\|\eta_{t_k}-\eta_{t_{k-1}}\|,
$$
where $\+H^{(1/2)}_{\eta_{t_{k-1}}}$ is the diagonal matrix with $i$th diagonal element equal to $\alpha_i^{1/2} \left.\dfrac{\partial^2 f_i}{\partial x^{(i)} y^{(i)}}\right\rvert^{1/2}_{(\eta_{t_{k-1}}^{(i)},\eta_{t_{k-1}}^{(i)})}$.

Summing over $k=1,\ldots,n$ and using the definition of $l(\eta)$ gives
$$
\sum_{k=1}^n  \|\+H^{(1/2)}_{\eta_{t_{k-1}}}(\eta_{t_k}-\eta_{t_{k-1}})\| \geq -\delta_1 +\sum_{k=1}^n  \|\zeta_{t_k}-\zeta_{t_{k-1}}\|.
$$
Combined with the relationship \eqref{eq:path_length} from Theorem~\ref{thm:generic} and yet another application of the triangle inequality we thus find
\begin{align*}
l(\gamma) & \geq - \epsilon -\delta_1 + \sum_{k=1}^n  \|\zeta_{t_k}-\zeta_{t_{k-1}}\|  \\
& \geq - \epsilon -\delta_1 + \| \zeta_1 - \zeta_0 \| \\
& = - \epsilon -\delta_1 + \| \psi\circ \phi^{-1}(b) - \psi \circ \phi^{-1}(a) \|. 
\end{align*}
The proof of the lower bound in the statement is complete since $\epsilon$ and $\delta_1$ are arbitrary. In order to complete the proof of the lemma, observe that \eqref{eq:additive_cont} combined with the same decomposition in \eqref{eq:zeta_decomp}-\eqref{eq:zeta_bound} yields an accompanying lower bound on $\zeta_{t_k}^{(i)} - \zeta_{t_{k-1}}^{(i)}$, from which it follows that 
$$
\left\lvert \sum_{k=1}^n  \|\+H^{(1/2)}_{\eta_{t_{k-1}}}(\eta_{t_k}-\eta_{t_{k-1}})\|  -  \sum_{k=1}^n  \|\zeta_{t_k}-\zeta_{t_{k-1}}\| \right\rvert \leq \delta_1.
$$
Substituting $\tilde{\zeta}$ as defined in the statement of the lemma in place of $\zeta$, and replacing $\eta$ by $\tilde{\eta}$ defined by $\tilde{\eta}_t\coloneqq \psi^{-1}(\tilde{\zeta}_t)$, then using the fact that $\delta_1$ is arbitrary we find via \eqref{eq:path_length} in  Theorem~\ref{thm:generic} that  $l(\tilde{\gamma})=\| \psi\circ \phi^{-1}(b) - \psi \circ \phi^{-1}(a) \|$.

\end{proof}

\section{Supplementary experiments} \label{appendix:experiments}
Codes for the experiments reported in the main part of the paper and those in this appendix are available at \url{https://github.com/anniegray52/graphs}. 

The R packages used in this paper are (with license details therein, see github repository for code):
data.table, RSpectra, igraph, plotly, Matrix, MASS, irlba, ggplot2, ggrepel, umap, Rtsne, lpSolve, spatstat, ggsci, cccd, R.utils, tidyverse, gridExtra, rgl, plot3D.
The Python modules used in this paper are (with license details therein, see github repository for code): networkx, pandas, nodevectors, random.

Data on the airports (e.g. their continent) was downloaded from \url{https://ourairports.com/data/}.

\subsection{Supplementary figures for simulated data example} \label{appendix:simulated}

\begin{figure}[h]
\centering
  \centering
  \includegraphics[width=\linewidth]{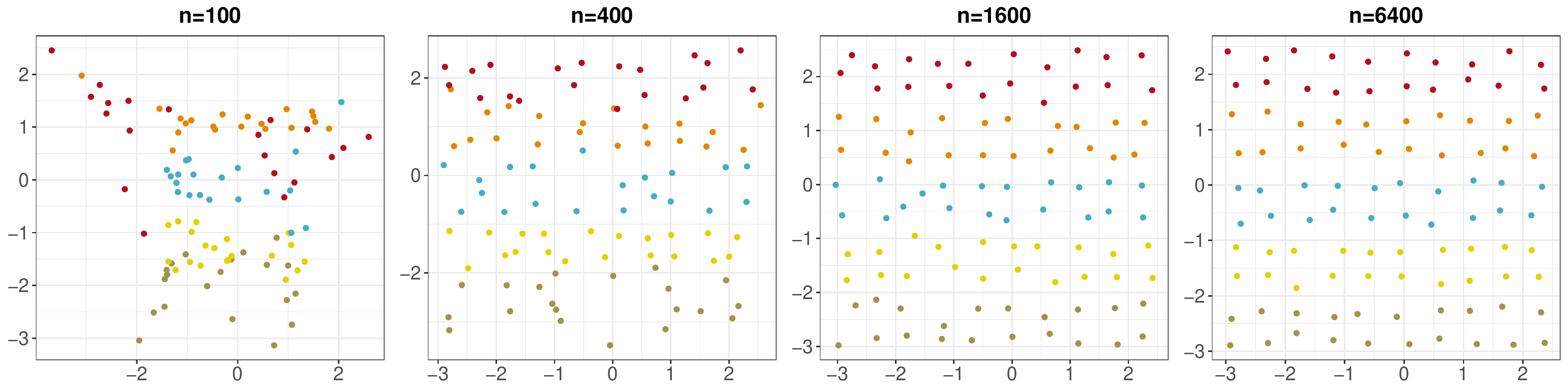}
  \caption{Latent position recovery in the sparse regime by spectral embedding followed by Isomap for increasing $n$ and increasing sparsity. To aid visualisation, all plots display a subset of $100$ estimated positions corresponding to true positions on a sub-grid which is common across $n$. Estimated positions are coloured according to their true $y$-coordinate.}
\label{fig:sparse_subset}
\end{figure}

\begin{figure}[h]
  \centering
  \includegraphics[width=\linewidth]{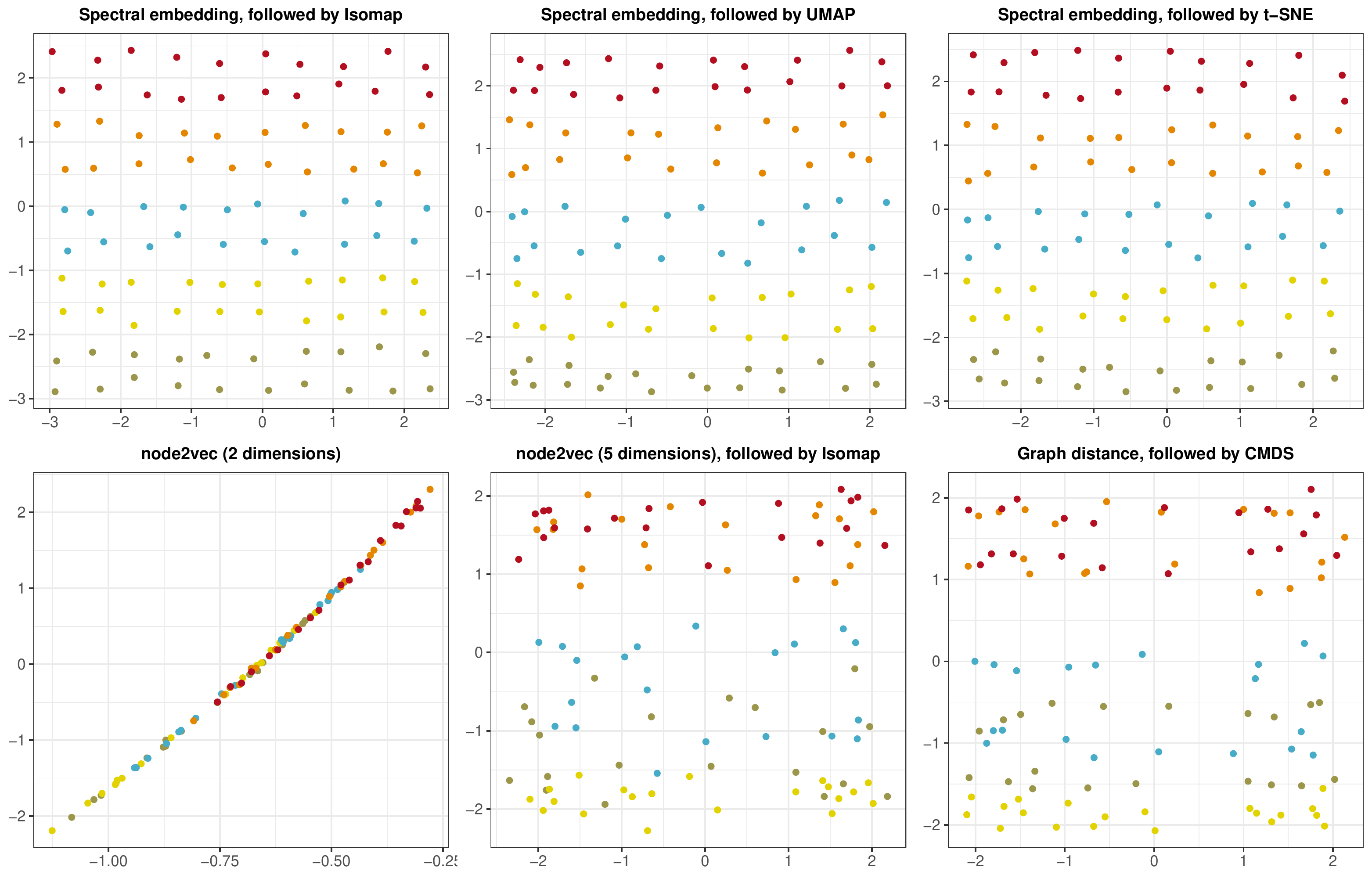}
  \caption{Latent position recovery in a sparse regime using different combinations of techniques with $n=6400$. The first row contains spectral embedding followed by different nonlinear dimension reduction techniques and the second row contains node2vec, node2vec followed by Isomap, and we have attempted recovery using graph distances \citep{10.1214/21-EJS1801}. To aid visualisation, all plots will display a subset of $100$  on a fixed sub-grid, coloured according to their true location.}
  \label{fig:method_comparison}
\end{figure}

\begin{figure}[h]
  \centering
  \includegraphics[width=\linewidth]{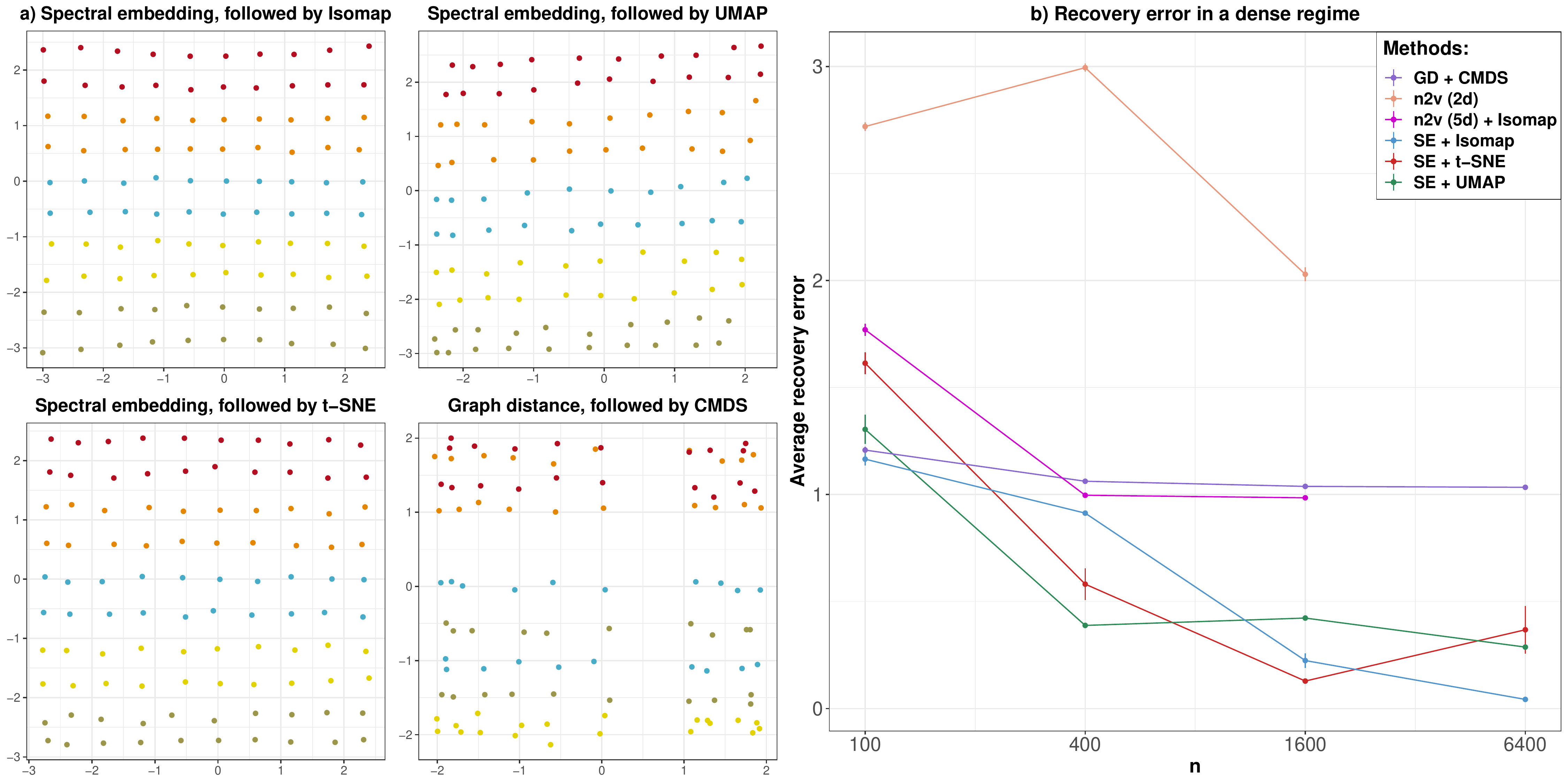}
  \caption{Latent position recovery in a dense regime. a) Recovery using different combinations of technique with $n=6400$. b) Average recovery error.  The recovery error is an average over nodes and over 100 simulations, with two standard errors shown as vertical bars. Computational issues precluded showing node2vec for $n=6400$. (SE = spectral embedding, GD = graph distance and n2v = node2vec.)}
  \label{fig:method_comparison_dense}
\end{figure}

\subsection{Supplementary figures and discussion for flight network example} \label{appendix:flight}

A gap appears to form between North and South America which, as a first hypothesis, we put down to the well-publicised suspension of all immigration into the US in April 2020. To measure this gap we use the Earth Mover's distance between the point clouds belonging to the two continents (with thanks to Dr. Louis Gammelgaard Jensen for the code), \emph{using the approximate geodesic distances} of the $\epsilon$-neighbourhood graph (i.e., before dimension reduction). While this distance does explode in April 2020, as shown in Figure~\ref{fig:Wassertein_NA-SA} (Appendix), we found the proposed explanation to be incomplete, because Australia and New Zealand imposed similar measures at the time, whereas the distance between Oceania and the rest of the world, computed in the same way, does not explode. Revisiting the facts \cite{borders}, while the countries mentioned above closed their borders to nonresidents, the continents of South America and Africa arguably imposed more severe measures, with large numbers of countries fully suspending flights. This explanation seems more likely, as on re-inspection we find a large jump in Earth mover's distance, over April, between \emph{both} of those continents and the rest of the world, as shown in Figure~\ref{fig:Wasserstein_all}. 

% \begin{figure}
%   \centering
%   \includegraphics[width=\textwidth]{Jan_latlong.pdf}
%   \caption{Visualisation of the flight network over January 2020, restricted to North America, showing only the airports within the continent's 5\%-95\% inter-quantile range of longitude and latitude. The colours indicate the longitude (a) and latitude (b) of the airports, showing that the technique recovers longitude relatively well.}\label{fig:jan_latlong}
% \end{figure}

\begin{figure}
  \centering
  \includegraphics[width=\textwidth]{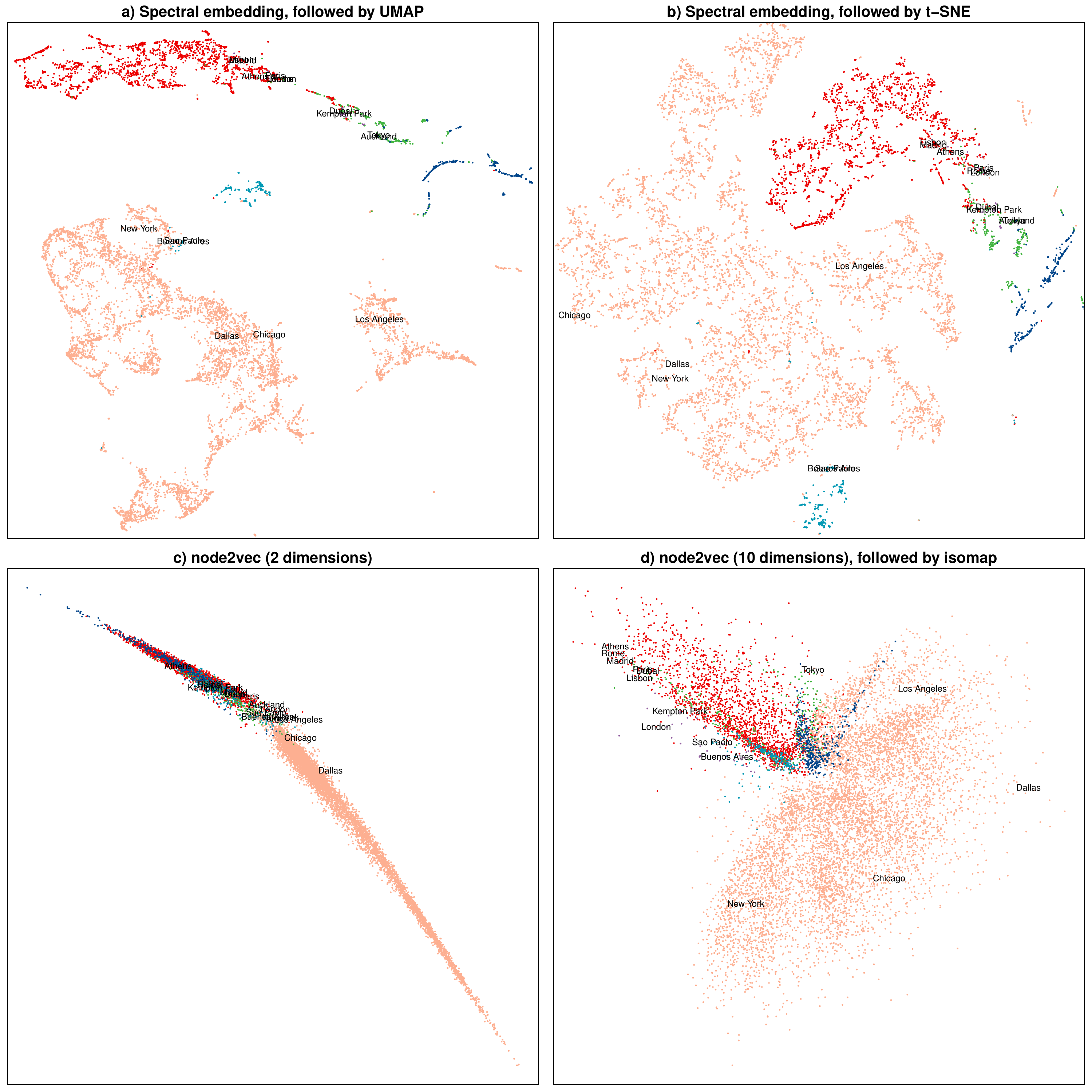}
  \caption{Visualisation of the global flight network over January 2020 by alternative combinations of techniques. The colours indicate continents (NA = North America, EU = Europe, AS = Asia, AF = Africa, OC = Oceania) and a spread of cities with high-traffic airports are labelled (to reduce clutter, only a selection of the cities in Figure~\ref{fig:Jan_isomap} are shown).}\label{fig:Jan_comparison_etc}
\end{figure}

\begin{figure}
    \centering
    \includegraphics[width=.5\textwidth]{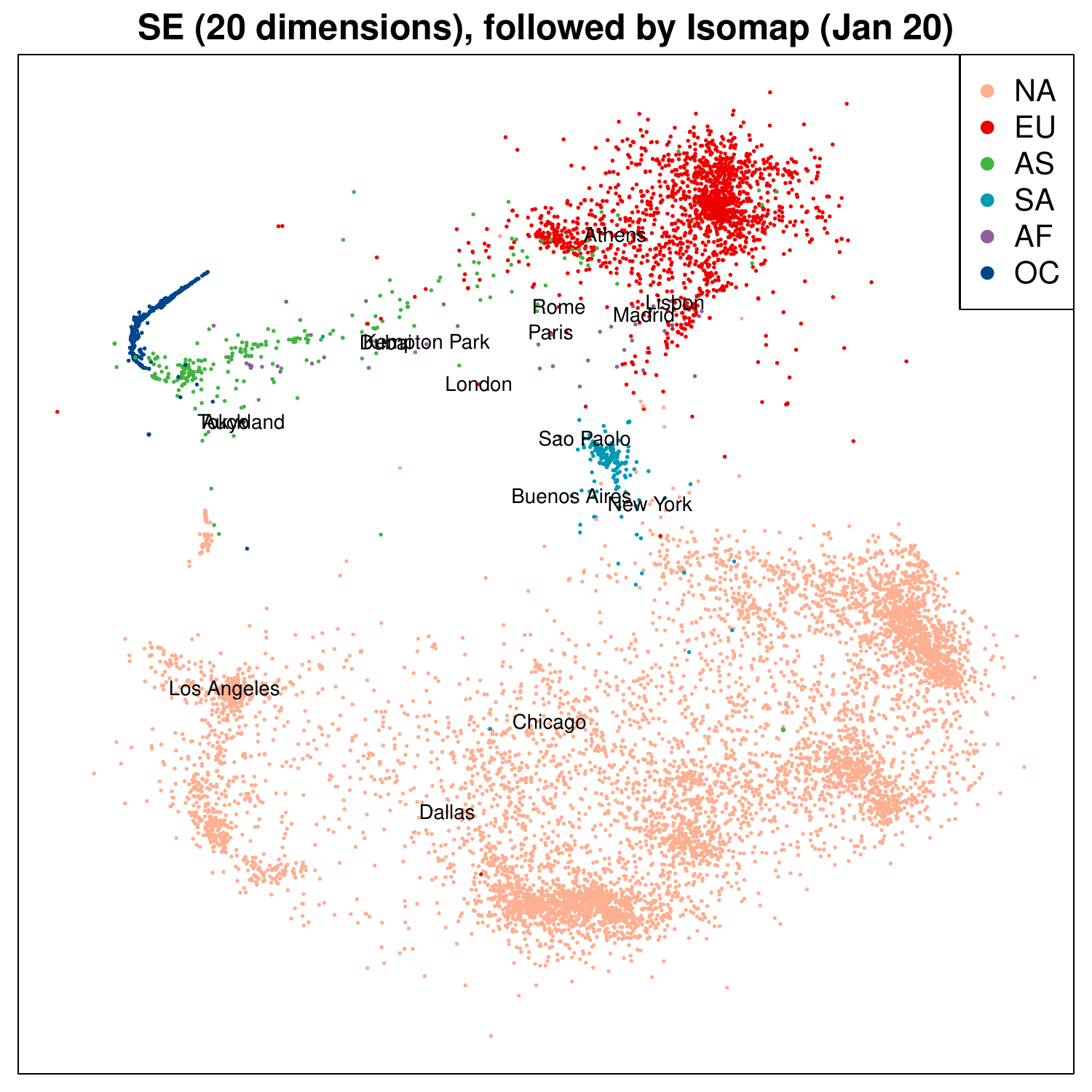}
    \caption{Visualisation of the global flight network. Spectral embedding into 20 dimensions followed by Isomap. The colours indicate continents (NA = North America, EU = Europe, AS = Asia, AF = Africa, OC = Oceania) and a spread of cities with high-traffic airports are labelled.}
    \label{fig:SE20}
\end{figure}

\begin{figure}
    \centering
    \includegraphics[width=.5\textwidth]{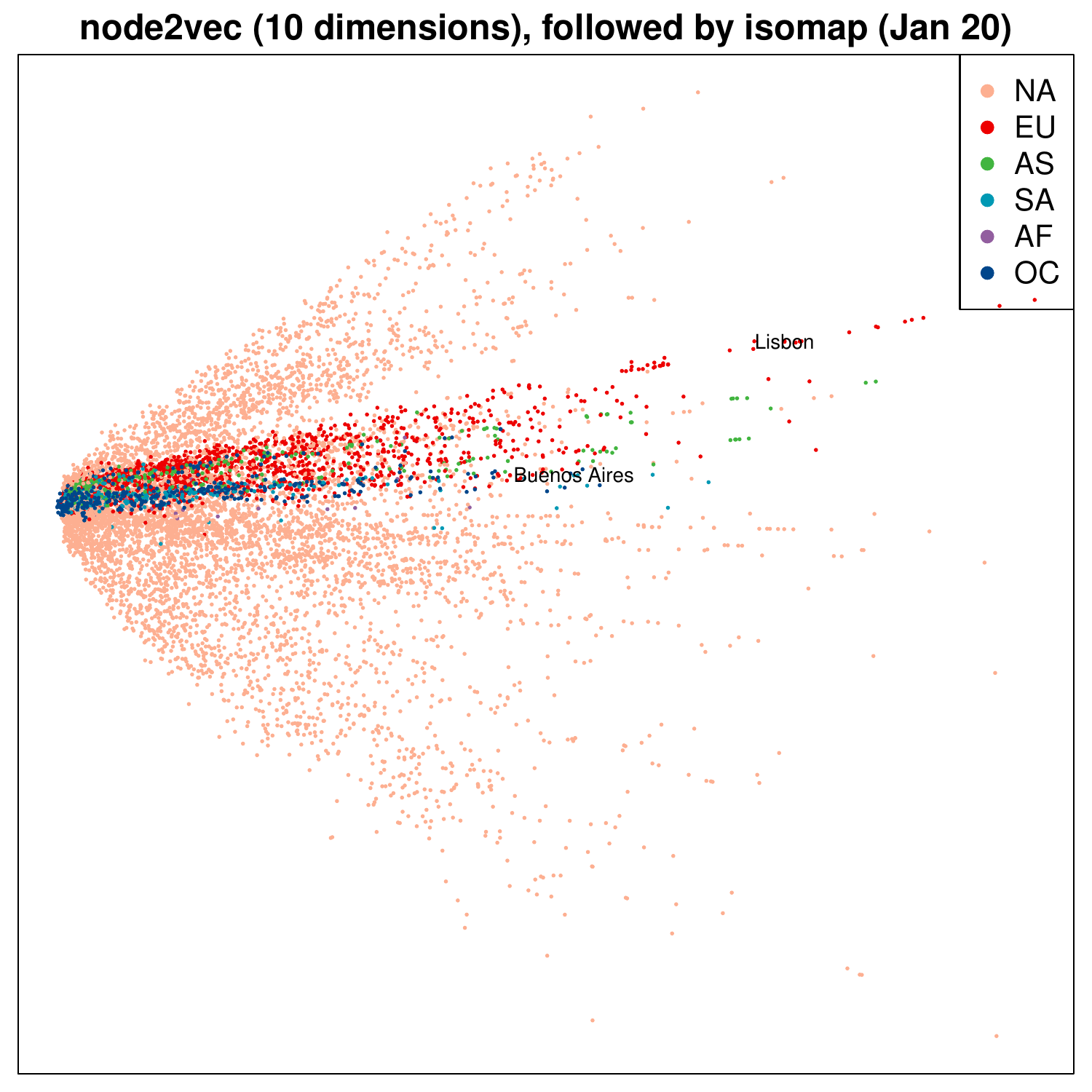}
    \caption{Visualisation of the global flight network. Node2vec into 20 dimensions followed by Isomap. The colours indicate continents (NA = North America, EU = Europe, AS = Asia, AF = Africa, OC = Oceania) and a spread of cities with high-traffic airports are labelled.}
    \label{fig:n2v20}
\end{figure}

\begin{figure}
  \centering
  \includegraphics[width=\textwidth]{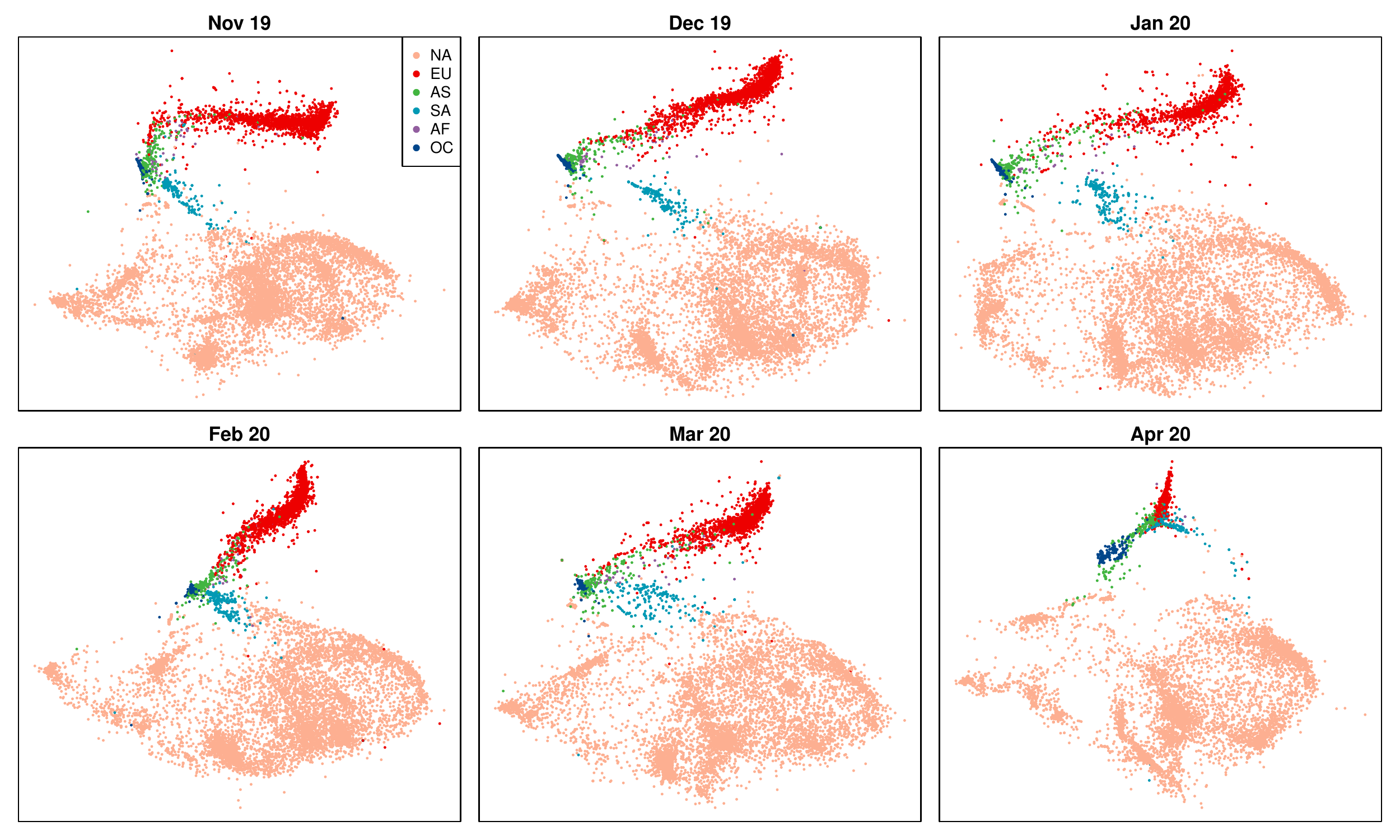}
  \caption{Visualisation of the global flight network over time: nonlinear dimension reduction of each spectral embedding using Isomap. The colours indicate continents (NA = North America, EU = Europe, AS = Asia, AF = Africa, OC = Oceania) and a spread of cities with high-traffic airports are labelled. An important structural change is observed in April 2020.}\label{fig:sixmonth}
\end{figure}

\begin{figure}
  \centering
  \includegraphics[width=.5\textwidth]{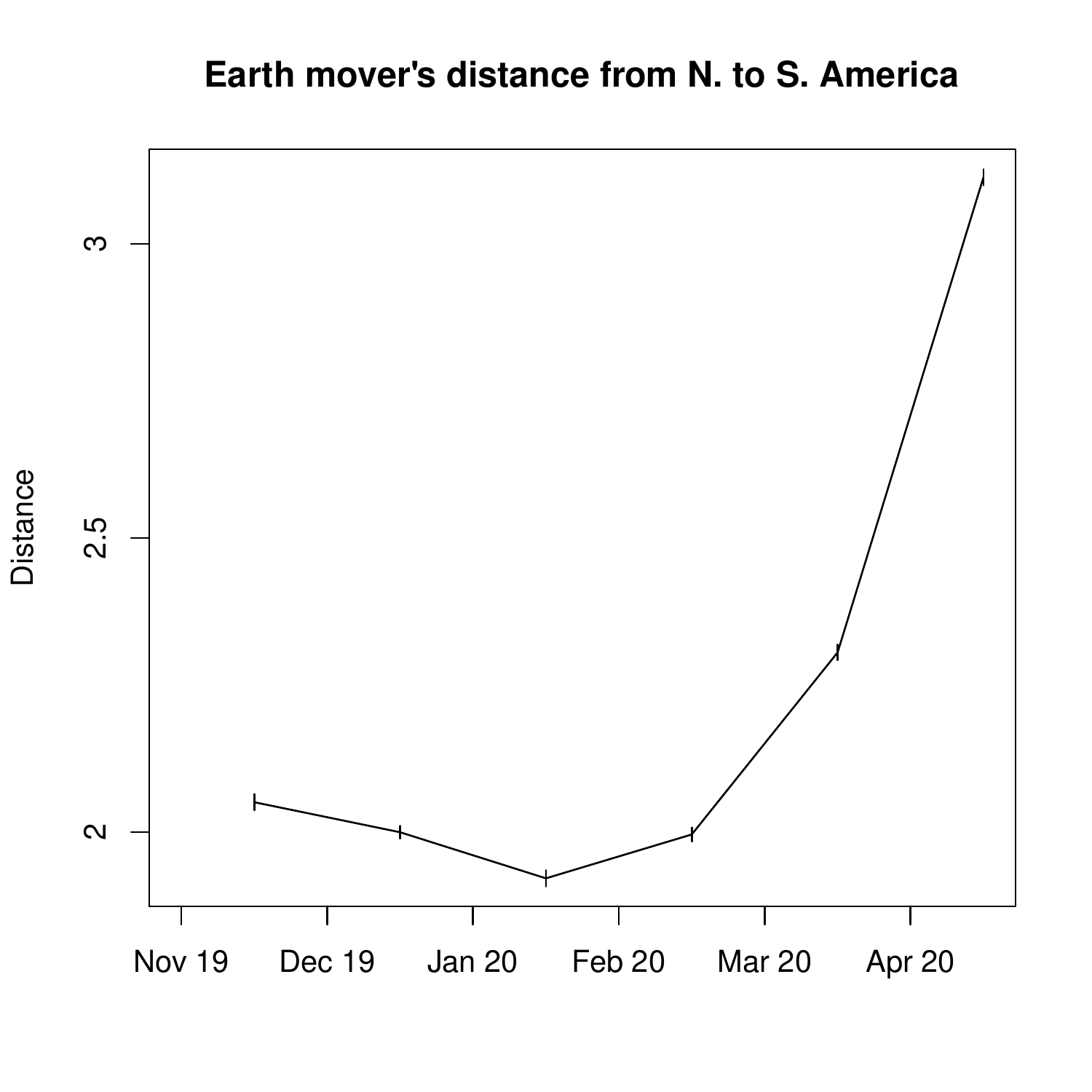}
  \caption{Earth mover's distance between North and South America, as inferred from the approximate geodesic distances given by the $\epsilon$-neighbourhood graph. In each of 100 Monte Carlo iterations, the Earth Mover's distance is computed based on 100 points randomly selected from each continent. We plot the average, with 2 standard errors in either direction indicated by the vertical bars.}\label{fig:Wassertein_NA-SA}
\end{figure}

\begin{figure}
  \centering
  \includegraphics[width=.5\textwidth]{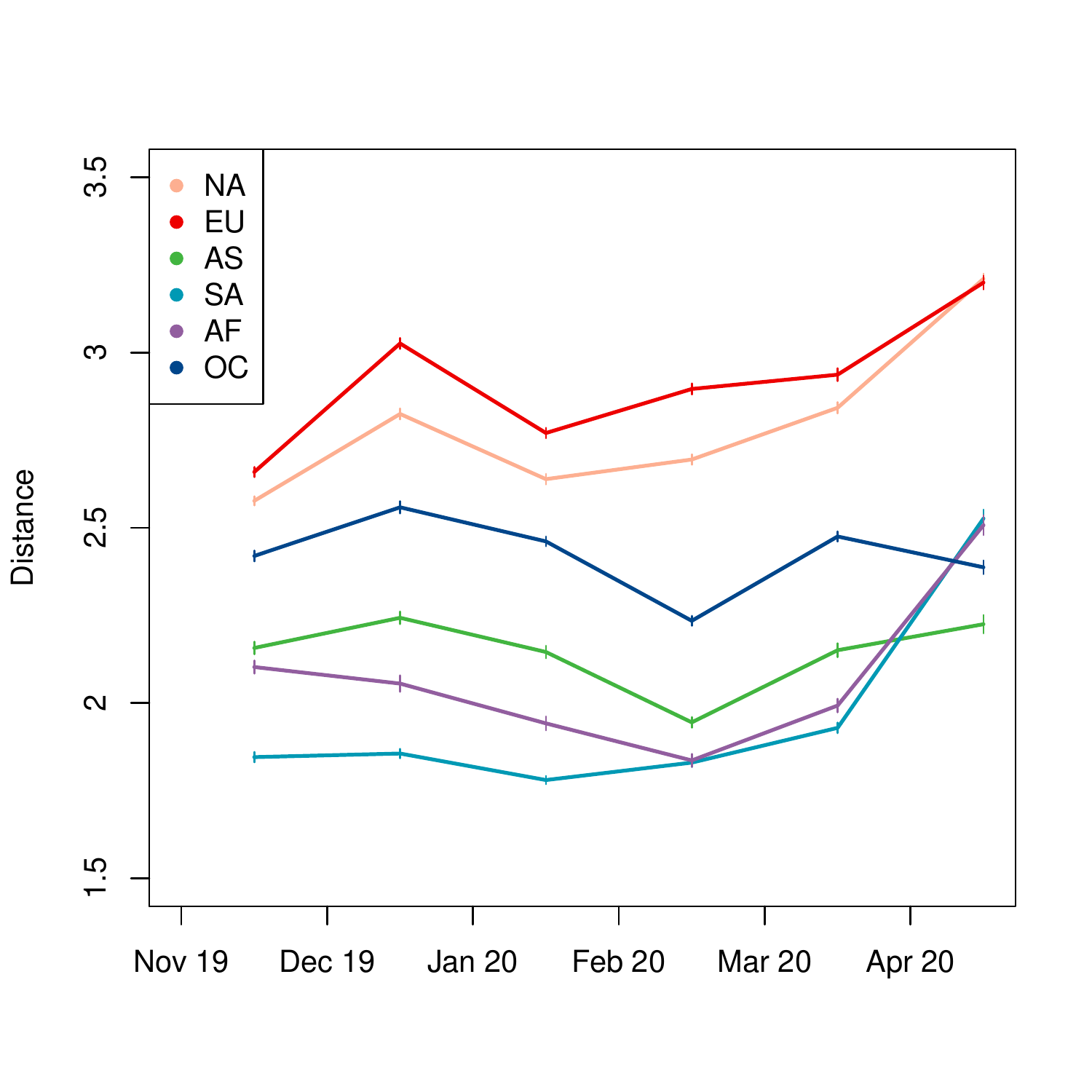}
  \caption{Earth mover's distance from each continent to the rest of the world, as inferred from the approximate geodesic distances given by the $\epsilon$-neighbourhood graph. In each of 100 Monte Carlo iterations, the Earth mover's distance is computed based on 100 points (airports) randomly selected from the continent of interest, and 100 points (airports) from the rest of the world. We plot the average, for each continent, with 2 standard errors in either direction indicated by the vertical bars.}\label{fig:Wasserstein_all}
\end{figure}

\subsection{Supplementary figures and discussion for temperature correlation example} \label{appendix:temperature}
Further to the discussion in Section~\ref{sec:experiments}, Figure \ref{fig:temp_d1_long} illustrates two important findings in the case $p=2$ and $d=1$, under the setup for this temperature correlation example described in section \ref{sec:experiments}: firstly that there is no clear relationship between longitude and position along the manifold, and secondly a non-monotone relationship between longitude and estimated latent position. Both these observations are in marked contrast with the results in Figure \ref{fig:temperature} for latitude.

We then consider the case $p=3$ and $d=2$. Figure \ref{fig:temp_p3_3d} shows the spectral embedding. Again the point-cloud is concentrated around a curved manifold. In Figure \ref{fig:temp_p3_zhat} we plot latitude and longitude against each of the $d=2$ coordinates of the estimated latent positions. As in the $p=2,d=1$ case, we find a clear monotone relationship between latitude and the first component of estimated position but no such relationship between longitude and the second component.
\begin{figure}
  \centering
  \includegraphics[width=\textwidth]{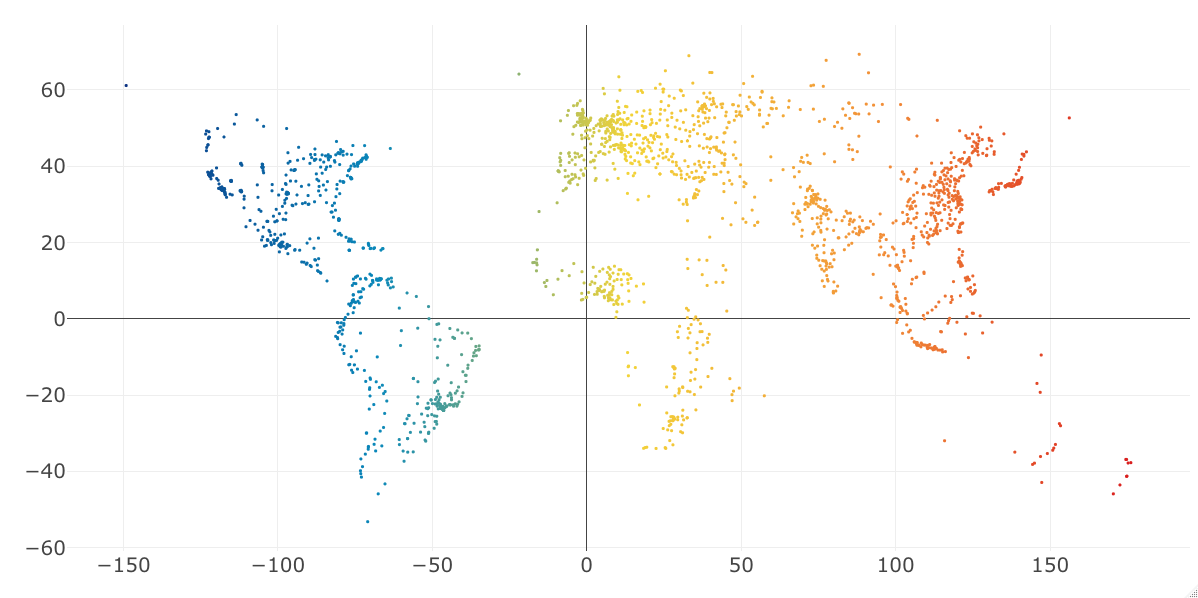}
    \includegraphics[width=\textwidth]{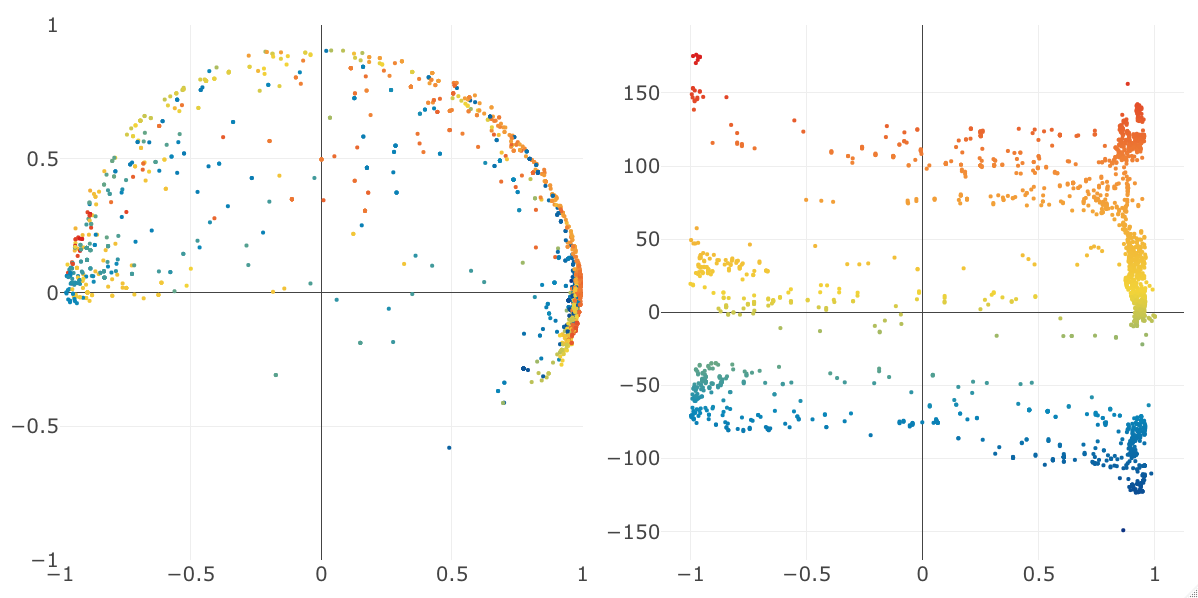}
    \caption{Temperature correlation example. Top: city locations. Bottom left: spectral embedding with $p=2$. Bottom right: true longitude (vertical axis) vs. estimated latent position (horizontal axis) with $d=1$. In all plots, points are coloured by longitudes of the corresponding cities.}\label{fig:temp_d1_long}
\end{figure}

\begin{figure}
  \centering
  \includegraphics[width=0.49\textwidth]{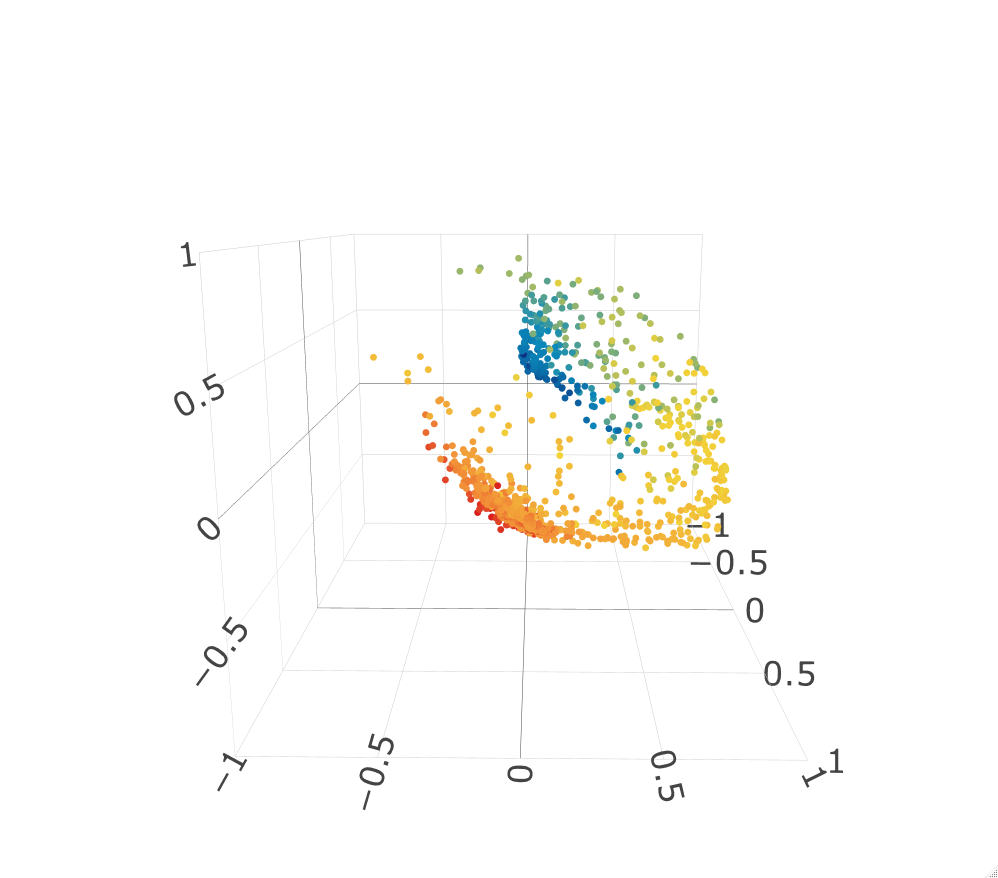}
  \includegraphics[width=0.49\textwidth]{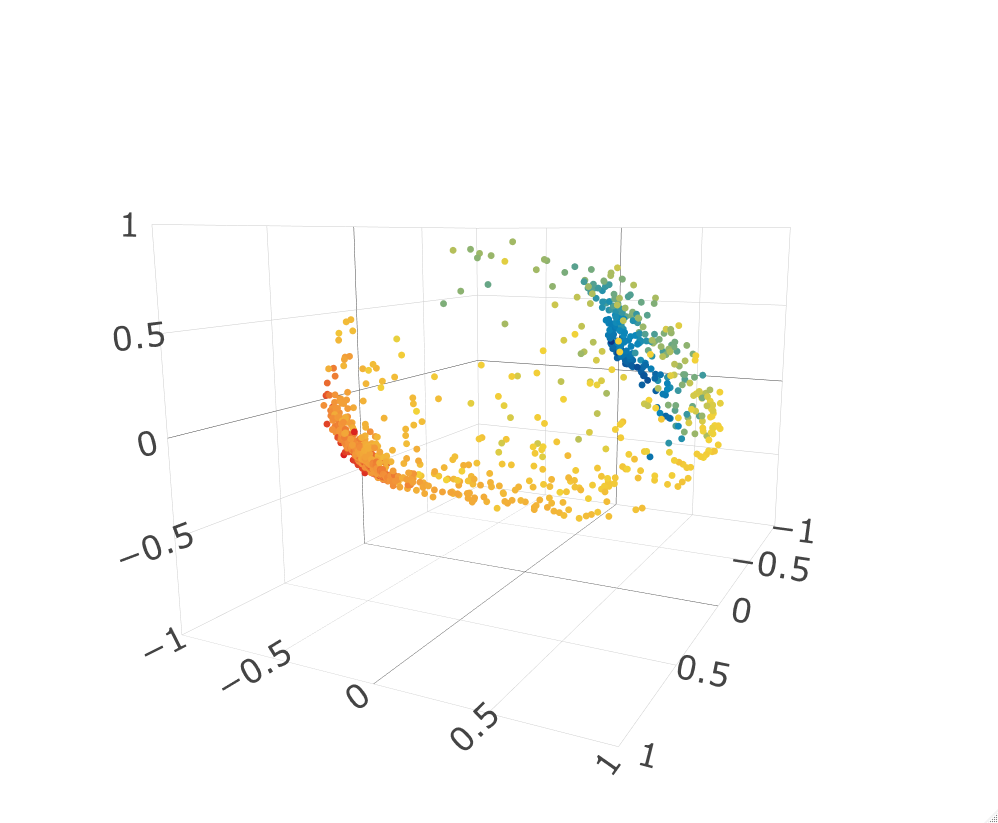}
  \includegraphics[width=0.49\textwidth]{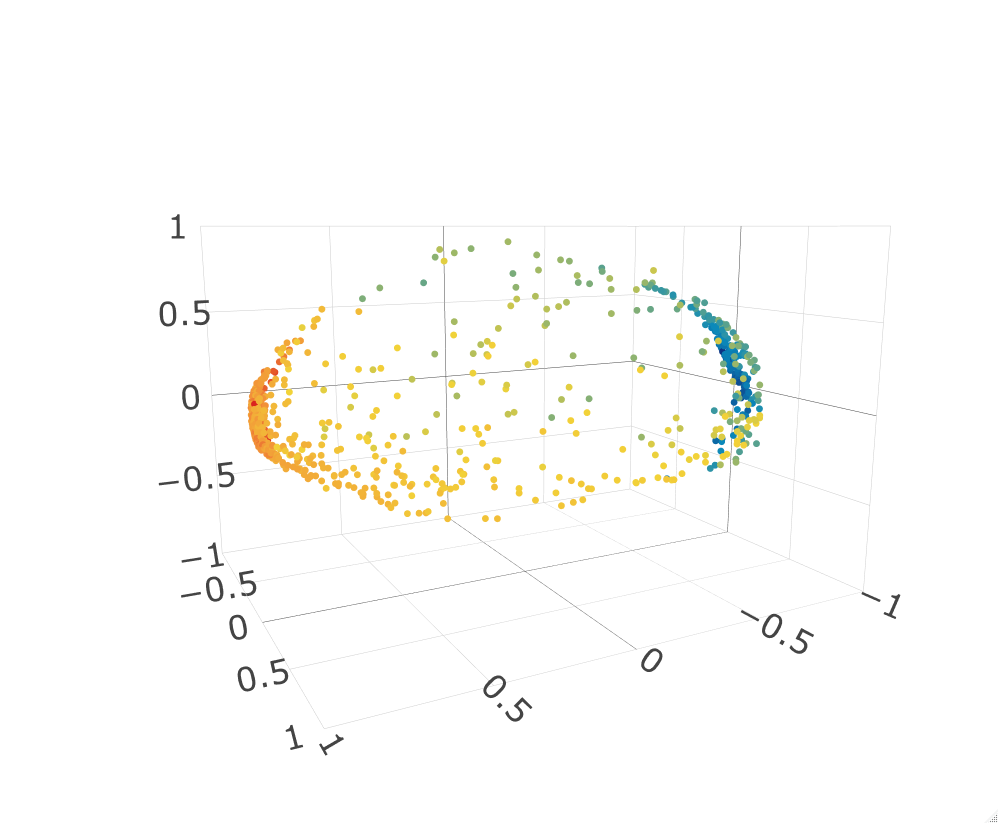}
  \includegraphics[width=0.49\textwidth]{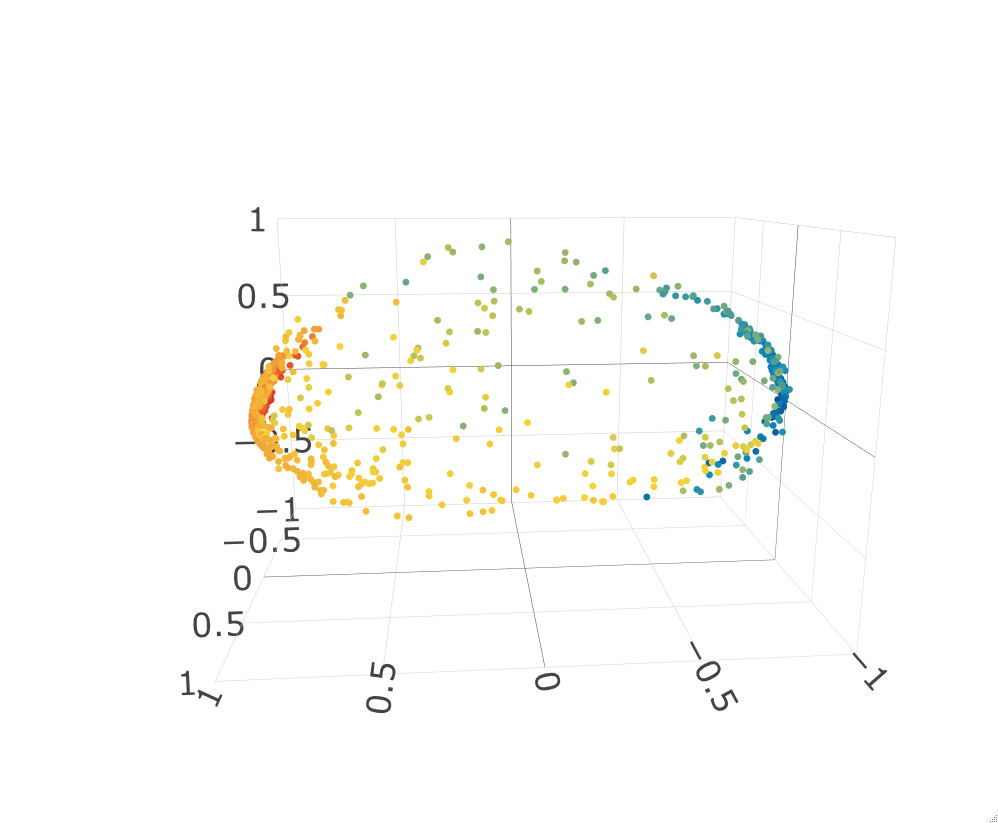}
  \caption{Temperature correlation example. Four views of the spectral embedding with $p=3$. In all plots, points are coloured by latitudes of the corresponding cities.}\label{fig:temp_p3_3d}
\end{figure}

\begin{figure}
  \centering
  \includegraphics[width=\textwidth]{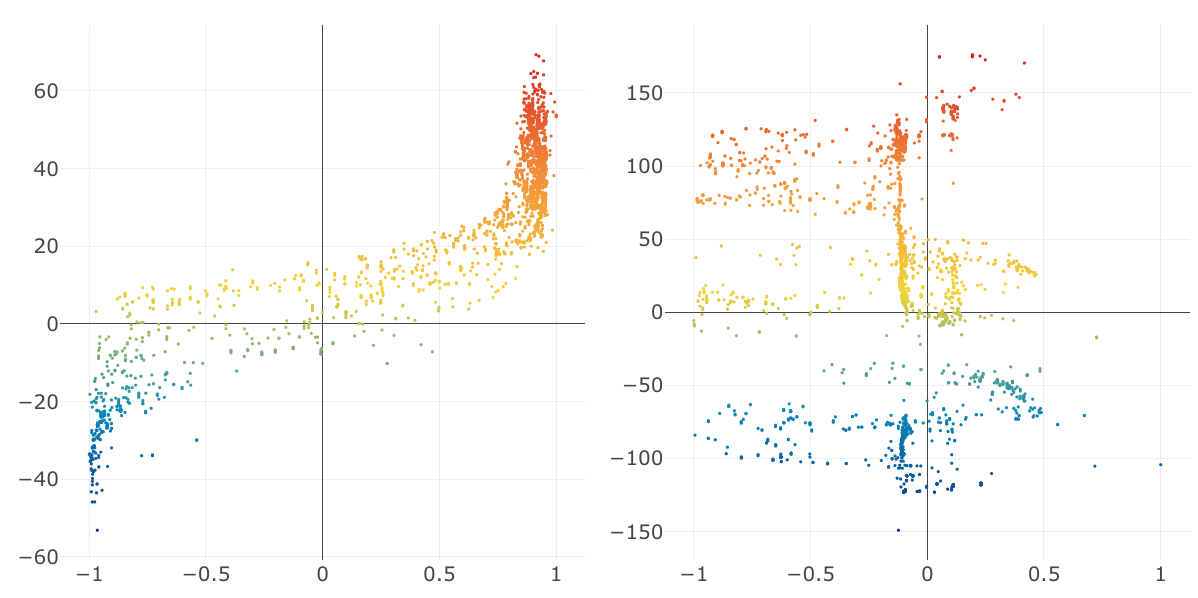}
  \caption{Temperature correlation example. Results for spectral embedding with $p=3$ followed by Isomap with $d=2$. Left: latitude (vertical axis) vs. estimated latent coordinate $\hat{Z}_i^{(1)}$ (horizontal axis) coloured by latitude. Right: longitude (vertical axis) vs. estimated latent coordinate $\hat{Z}_i^{(2)}$ (horizontal axis) coloured by longitude. }  \label{fig:temp_p3_zhat}
\end{figure}

\end{document}